\documentclass{article} 
\usepackage{iclr2026_conference,times}

\newcommand{\anonpublic}[2]{%
  \ificlrfinal
    #2%
  \else
    #1%
  \fi
}

\usepackage{amsmath,amsfonts,bm}

\def\eqref#1{equation~\ref{#1}}

\def\1{\bm{1}}

\DeclareMathAlphabet{\mathsfit}{\encodingdefault}{\sfdefault}{m}{sl}
\SetMathAlphabet{\mathsfit}{bold}{\encodingdefault}{\sfdefault}{bx}{n}

\usepackage{hyperref}
\usepackage{url}
\usepackage{graphicx}
\usepackage{wrapfig}
\usepackage{epigraph}
\usepackage{amsmath}
\usepackage{amssymb}
\usepackage{graphicx}
\usepackage{mathtools}
\usepackage{bm}
\usepackage{enumitem}
\usepackage{caption}
\usepackage{arydshln}

\usepackage{amsthm}
\usepackage[ruled,vlined]{algorithm2e}  

\title{Stabilizing Policy Gradients for \\ Sample-Efficient Reinforcement Learning \\ in LLM Reasoning}

\author{Luckeciano C. Melo$^{1,2}$\enskip\enskip
Alessandro Abate$^{2}$\enskip\enskip
Yarin Gal$^{1}$ \\
$^1$OATML, University of Oxford\enskip\enskip$^2$OXCAV, University of Oxford\\
\texttt{luckeciano.carvalho.melo@cs.ox.ac.uk}
}

\newcommand{\anonurl}{\url{https://anonymous.4open.science/r/capo-stable-gradients}}
\newcommand{\publicurl}{\url{https://github.com/luckeciano/stable-pg-llm}}

\newcommand{\codeurl}{\anonpublic{\anonurl}{\publicurl}}

\theoremstyle{plain}
\newtheorem{theorem}{Theorem}[section]
\newtheorem{lemma}{Lemma}[section]
\newtheorem{proposition}{Proposition}[section]

\theoremstyle{definition}

\newtheorem{assumption}{Assumption}[section]

\newcommand{\norm}[1]{\left\lVert #1 \right\rVert}
\DeclareMathOperator{\op}{op}

\iclrfinalcopy
\begin{document}

\maketitle
\begin{abstract}
Reinforcement Learning, particularly through policy gradient methods, has played a central role in enabling reasoning capabilities of Large Language Models. However, the optimization stability of policy gradients in this setting remains understudied. As a result, existing implementations often resort to conservative hyperparameter choices to ensure stability, which requires more training samples and increases computational costs. Hence, developing models for reliably tracking the underlying optimization dynamics and leveraging them into training enables more sample-efficient regimes and further unleashes scalable post-training. We address this gap by formalizing the stochastic optimization problem of policy gradients with explicit consideration of second-order geometry. We propose a tractable computational framework that tracks and leverages curvature information during policy updates. We further employ this framework to design interventions in the optimization process through data selection. The resultant algorithm, Curvature-Aware Policy Optimization (CAPO), identifies samples that contribute to unstable updates and masks them out. Theoretically, we establish monotonic improvement guarantees under realistic assumptions. On standard math reasoning benchmarks, we empirically show that CAPO ensures stable updates under aggressive learning regimes where baselines catastrophically fail. With minimal intervention (rejecting fewer than $8\%$ of tokens), CAPO achieves up to 30$\times$ improvement in sample efficiency over standard GRPO for LLM reasoning.
\end{abstract}

\begin{wrapfigure}{r}{0.5\textwidth}
\vspace{-65pt}
\begin{center}
\includegraphics[width=0.49\textwidth]{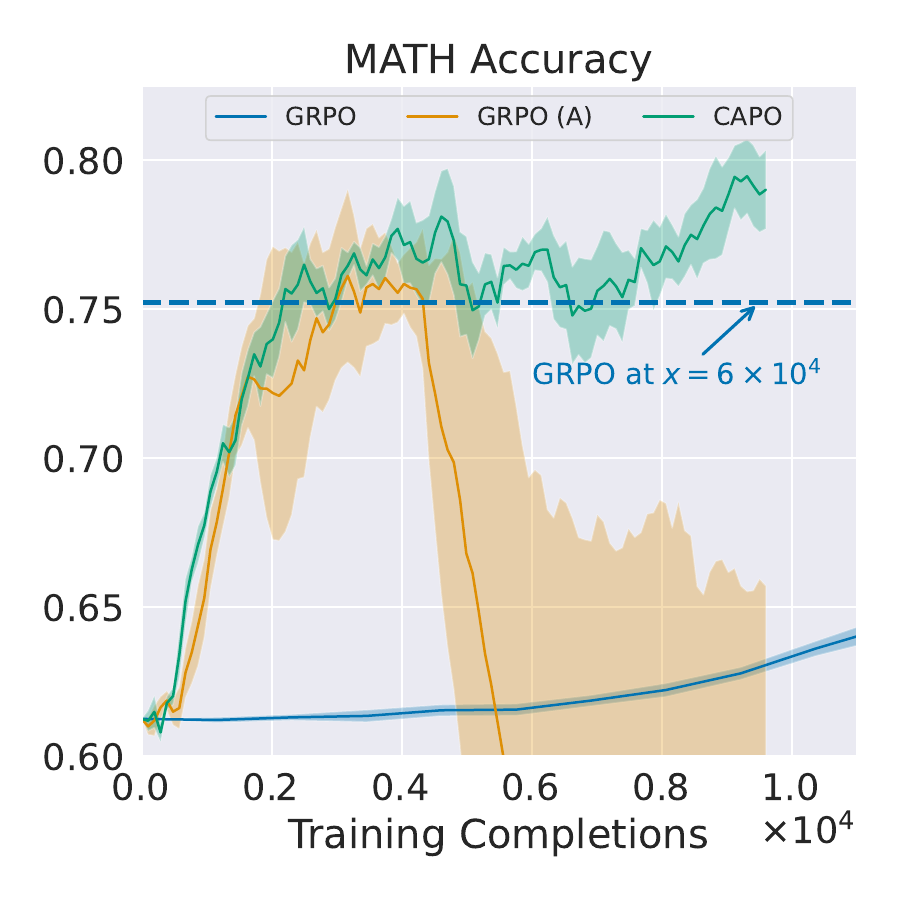}
\captionsetup{skip=-1pt}
\caption{\textbf{Accuracy on MATH dataset from different RL methods.} 
CAPO (ours) achieves $30\times$ greater sample efficiency under an aggressive (A) update regime 
(higher learning rate, smaller batch size), whereas GRPO suffers policy collapse.}
\label{fig:math_accuracy}
\end{center}
\end{wrapfigure}

\section{Introduction}

The emergence of reasoning capabilities in Large Language Models (LLMs) represents a major shift in AI research. Beyond language understanding, reasoning has become a core ingredient of widely deployed systems \citep{openai2024openaio1card, comanici2025gemini25pushingfrontier}, enabling applications such as mathematical problem solving \citep{deepseek-math}, code generation \citep{shojaee2023executionbased}, and agentic workflows \citep{yao2023react}. This progress is primarily attributed to advances in scaling Reinforcement Learning (RL) techniques for LLM post-training, particularly policy gradient methods such as PPO \citep{schulman2017proximalpolicyoptimizationalgorithms}, GRPO \citep{deepseek-math}, and variants \citep{yu2025dapoopensourcellmreinforcement, liu2025understandingr1zeroliketrainingcritical}. These methods enabled LLMs to develop behaviors for autonomous chain-of-thought reasoning \citep{gandhi2025cognitivebehaviorsenableselfimproving} and to effectively scale test-time compute \citep{setlur2025scalingtesttimecomputeverification}.

Despite its success in LLM fine-tuning and other decision-making tasks \citep{baloonnature, mnih2015humanlevel}, RL still faces fundamental challenges that limit its broader practicality and scalability. In particular, policy gradients suffer from optimization instabilities  driven by the non-stationary nature of the RL objective and the high variance of estimates \citep{castanyer2025stablegradientsstablelearning}. These problems are further compounded by the known pathologies of training deep networks \citep{pmlr-v28-pascanu13, NIPS2017_d9fc0cdb}. These factors lead to several undesired consequences, such as catastrophic updates and policy collapse \citep{dohare2023overcoming}, plasticity loss \citep{NEURIPS2024_ce7984e3}, sample inefficiency \citep{Kaiser2020Model}, and hyperparameter sensitivity \citep{Henderson_Islam_Bachman_Pineau_Precup_Meger_2018}. As a result, the optimization dynamics of RL remain an active area of research from both theoretical and empirical standpoints \citep{NEURIPS2022_718d02a7, pmlr-v162-lyle22a, pmlr-v151-vaswani22a} .

Perhaps due to the recency of the topic, the optimization dynamics of RL \textit{in the context of LLMs} remains underexplored. These challenges persist in the LLM setting and may be even more pronounced, since training involves billion-parameter models with very deep architectures and sampling horizons that can extend arbitrarily. In practice, current implementations of RL for LLMs typically rely on conservative hyperparameters to ensure stability, such as low learning rates (e.g., $3\times10^{-6}$ or less) and large batch sizes (e.g. thousands of generations per policy update) \citep{sheng2024hybridflow, openr1, Guo2025}. These choices substantially increase the number of LLM generations required for learning, raising computational costs. Therefore, stabilizing these algorithms in sample-efficient regimes is crucial to further scale RL for LLM reasoning.

One promising direction is to design algorithms that explicitly model second-order geometry in the optimization landscape and incorporate this information into policy updates. In this work, we formalize the RL optimization problem accounting for curvature terms, namely the Hessian of the objective and the Fisher Information Matrix of the policy distribution. Building on this formulation, we introduce a computationally and numerically tractable model of optimization dynamics that approximates this curvature information. This model enables continuous monitoring of gradient and curvature estimates during policy updates, scales to billion-parameter models and provides analytical expressions for these quantities, which facilitate a systematic analysis of the learning dynamics.

We further leverage this optimization model to \textit{plan} the next policy gradient step\footnote{In this work, “model” refers to the proposed computational model of curvatures and “policy” to the LLM. “Model gradients” are computed under the former, while “policy gradients” denote the true LLM gradients.}. It allows \textit{anticipating} policy updates that potentially induce sudden shifts in the objective or policy distribution --  often associated with unstable optimization behavior -- and intervening before taking the actual step in the LLM. We propose a simple data selection mechanism as intervention: we identify particular samples that heavily contribute to these abrupt shifts and mask them out of the policy gradient estimation. We refer to this method as \textit{Curvature-Aware Policy Optimization} (CAPO). 

We theoretically establish monotonic policy improvement guarantees under CAPO with practical assumptions. We then empirically validate CAPO on standard math reasoning benchmarks, showing that it yields stable optimization even in regimes with aggressive updates, where standard RL algorithms suffer catastrophic updates and policy collapse. As a result, CAPO achieves up to $30\times$ improvement in sample efficiency compared to GRPO in the standard regime, as presented in Figure \ref{fig:math_accuracy}. Lastly, we show that its interventions are minimal, typically rejecting fewer than $8\%$ of the tokens, with negligible computational overhead. 
\section{Related Work}
\textbf{RL \& LLMs.} The use of RL techniques to optimize LLMs has been an active area of research in recent years. Early work focused on RL from Human Feedback (RLHF), which optimizes policies toward modeled human preferences \citep{ziegler2019finetuning, NEURIPS2020_1f89885d, 10.5555/3600270.3602281}. More recently, RL for LLM reasoning has gained significant attention for its effectiveness in enabling autonomous chain-of-thought reasoning \citep{gandhi2025cognitivebehaviorsenableselfimproving} and in scaling test-time compute \citep{setlur2025scalingtesttimecomputeverification}. This breakthrough was initially driven by the seminal works of the OpenAI o-series \citep{openai2024openaio1card} and DeepSeek-R1 \citep{Guo2025}, which popularized GRPO \citep{deepseek-math}. Since then, the research community has studied several aspects of the training pipeline \citep{zhang2025surveyreinforcementlearninglarge}, including alternative objectives \citep{roux2025taperedoffpolicyreinforcestable, hu2025reinforceefficientrlhfalgorithm}, sampling mechanisms \citep{yu2025dapoopensourcellmreinforcement}, reward shapings \citep{yang2024qwen25mathtechnicalreportmathematical}, and different training configurations \citep{liu2025understandingr1zeroliketrainingcritical, kimiteam2025kimik15scalingreinforcement}. Our work fits within this line of research by investigating RL for LLMs from an optimization dynamics perspective, proposing a model of the optimization landscape and using it to design stable policy gradient updates.

\textbf{Optimization Dynamics in RL.} The non-convex and non-stationary nature of RL training has motivated a large body of work on understanding and stabilizing optimization dynamics in RL agents. In the context of policy gradients, prior research has investigated the role of baselines \citep{NEURIPS2022_718d02a7}, variance reduction techniques \citep{10.5555/2980539.2980735}, and emergent pathologies such as plasticity or capacity loss \citep{pmlr-v202-sokar23a, klein2024plasticitylossdeepreinforcement} and policy collapse \citep{dohare2023overcoming}. Beyond these analyses, past literature has also developed conservative policy optimization methods \citep{pmlr-v37-schulman15, schulman2017proximalpolicyoptimizationalgorithms,pmlr-v70-achiam17a}. While this line of work is extensive and evolving, we primarily highlight the recent contribution of \citeauthor{castanyer2025stablegradientsstablelearning}, which, like ours, examines the stabilization of policy gradients through curvature-informed interventions. Their methodology, however, differs: they apply natural gradients with K-FAC \citep{NEURIPS2023_6a6679e3} in general deep RL environments, whereas our work develops a new approximation of curvature that is tractable at the scale of LLMs and is incorporates it into optimization through data selection.

\textbf{Improving  RL for LLM Reasoning.}
In the context of LLM research, a nascent but growing literature explores improvements to RL training for reasoning. These works typically propose heuristics that target specific problems observed during training—for example, noisy gradient estimates, limited output diversity, or large policy updates. Common approaches include rethinking advantage estimation \citep{liu2025understanding, ahmadian-etal-2024-back}, controlling policy entropy \citep{yu2025dapoopensourcellmreinforcement,cui2025entropymechanismreinforcementlearning}, and bounding advantage estimates or log-likelihoods \cite{yang2025dcpodynamicclippingpolicy, yang2025AR_Lopti}. In contrast, our work takes a more principled approach. Rather than introducing heuristics to address isolated issues, we develop a framework based on second-order stochastic optimization that fundamentally explains these instabilities and addresses them in a unified manner.

\section{Preliminaries}

\textbf{Problem Statement.} We formulate the problem of next-token generation as a Markov Decision Process (MDP), defined by the tuple $\mathcal{M} = (\mathcal{S}, \mathcal{A}, \mathcal{P}, R, \rho_0, \gamma, T)$, in which $\mathcal{S}$ is a state space, $\mathcal{A}$ is an action space, $\mathcal{P} : \mathcal{S} \times \mathcal{A} \to \Delta(\mathcal{S})$ a transition function, $R : \mathcal{S} \times \mathcal{A} \to [-r_{\text{bound}}, +r_{\text{bound}}]$ a bounded reward function, $\rho_0 : \mathcal{S} \to \Delta(\mathcal{S})$ an initial state distribution, $\gamma \in [0, 1]$ a discount factor, and $T$ the length of the horizon. In the LLM setting, let $\mathcal{V}$ be a token vocabulary and $L \in \mathbb{N}$ a maximum sequence length, including both prompt and generated tokens. $\mathcal{S} = \bigcup_{n=0}^{L} \mathcal{V}^n$ is the set of all finite sequences, with each state $s_t \in \mathcal{S}$  representing the concatenation of the prompt and the tokens generated up to time $t$, with total length at most $L$. $\mathcal{A}$ is the space spanned by $\mathcal{V}$: at each step, the policy selects a token $a_t \in \mathcal{V}$. $\mathcal{P}$ is governed by autoregressive sampling and takes the form of a trivial deterministic function $s_{t+1} = s_t \circ a_t$, where $\circ$ denotes concatenation. The initial state distribution $\rho_0$ specifies a distribution over prompts. During policy optimization, one typically optimizes a parameterized LLM $\pi_{\boldsymbol{\theta}} : \mathcal{S} \times \mathcal{A} \to \Delta(\mathcal{A})$, with the objective of maximizing the expected cumulative reward over the generated sequence:
\begin{equation}\label{eq:obj}
    J(\boldsymbol{\theta}) = \mathbb{E}_{\tau \sim \pi_{\boldsymbol{\theta}}} \Big[ \sum_{t=0}^T \gamma^t R(s_t, a_t) \Big],
\end{equation} 
where $\tau$ denotes a trajectory, $s_0 \sim \rho_0(s_0)$, $a_t \sim \pi_{\boldsymbol{\theta}}(a_t \mid s_t)$, and $s_{t+1} = s_t \circ a_t$.

\textbf{Policy Gradient} (PG) methods optimize a stochastic policy by differentiating $J(\boldsymbol{\theta})$ with respect to the policy parameters \citep{Williams:92} and can be written as \citep{10.5555/3009657.3009806}:
\begin{equation}
    \nabla_{\boldsymbol{\theta}} J(\boldsymbol{\theta}) 
= \mathbb{E}_{\tau \sim \pi_{\boldsymbol{\theta}}} \left[ \sum_{t=0}^T \gamma^{t} \nabla_{\boldsymbol{\theta}} \log \pi_{\boldsymbol{\theta}}(a_t \mid s_t)\, R(s_t, a_t) \right].
\end{equation}
This expectation can be estimated via Monte Carlo sampling under the current policy $\pi_{\boldsymbol{\theta}}$. However, such estimates often have high variance. A standard remedy is to subtract a baseline $b(s_t)$ which leaves the gradient unbiased while reducing variance. In practice, this is typically done by replacing the reward with an estimate of the advantage function $A(s_t, a_t)$. For the rest of this work, we will assume the advantage version of this objective.

\textbf{Group Relative Policy Optimization}  \citep{deepseek-math} is a widely used method for RL in LLMs. Akin to PPO \citep{schulman2017proximalpolicyoptimizationalgorithms}, it optimizes a surrogate objective that employs off-policy correction \cite{Kakade2002ApproximatelyOA} with a clipping strategy to prevent large deviations:
\begin{align}\label{eq:grpo}
    J_{\text{GRPO}}(\boldsymbol{\theta}) = \mathbb{E}_{\tau \sim \pi_{\beta}} \Big[ \frac{1}{|\tau_{i}|}\sum_{t=0}^{|\tau_{i}|}&\min\Big(r_{\theta}(s_t, a_t), \text{clip}(r_{\boldsymbol{\theta}}(s_t, a_t), 1 - \epsilon, 1 + \epsilon)\Big) A^{\text{GRPO}}(s_{t}, a_{t})  \nonumber \\
    -& \beta \mathcal{D}_{\text{KL}}\!\left(\pi_{\boldsymbol{\theta}}(\cdot \mid s_t) \,\|\, \pi_{\text{base}}(\cdot \mid s_t)\right) \Big],
\end{align}
where $r_{\boldsymbol{\theta}}(s_t, a_t) = \frac{\pi_{\boldsymbol{\theta}}(a_t \mid s_t)}{\pi_{\beta}(a_t \mid s_t)}$ and $\pi_{\beta}$ is the sampling policy. The KL divergence term acts as a regularizer that penalizes deviation from $\pi_{\text{base}}$, the initial LLM. In contrast to standard PG methods, GRPO draws samples in groups: for each prompt $s_0 \sim \rho_0(s_{0})$, it generates a group of trajectories $\{\tau_i\}_{i=1}^G \sim \pi_\beta$. Contributions from all state-action pairs of a trajectory are averaged (rather than discounted), which effectively assume $\gamma = 1$ with per-trajectory normalization. Finally, the advantage estimator is defined as:
\begin{equation}
\hat{A}^{\text{GRPO}}(s_t, a_t)
\;=\;
\frac{ \hat{R}(\tau) - \bar{R} }{ \hat{\sigma}_R + \varepsilon },
\quad \bar{R} \;=\; \frac{1}{G} \sum_{i=1}^G \hat{R}(\tau_i),
\quad
\hat{\sigma}_R \;=\; \sqrt{ \frac{1}{G} \sum_{i=1}^G \big(\hat{R}(\tau_i) - \bar{R}\big)^2 }, 
\end{equation}
where $\hat{R}(\tau)$ is the return for trajectory $\tau$ and $\varepsilon$ is a small constant for numerical stability.

\section{Modeling the Optimization Landscape with Second-Order Geometry}

In this section, we develop a model of the optimization landscape. We formulate the reinforcement learning (RL) optimization problem with policy gradients by explicitly incorporating second-order geometric information. Building on this formulation, we introduce a tractable computational model that approximates the role of curvature during learning. Our hypothesis is that by providing a simple but effective approximation of second-order gradients, one could track sudden shifts in the objective or policy and anticipate potentially unstable updates.

\textbf{The Higher-Order Objective}. Consider the objective function $J(\boldsymbol{\theta})$ as in Equation \ref{eq:obj}. After an update step $\Delta\boldsymbol{\theta}$, the new objective $J(\boldsymbol{\theta} + \Delta\boldsymbol{\theta})$ is given by the following Taylor expansion:
\begin{equation}\label{eq:taylor}
    J(\boldsymbol{\theta} + \Delta\boldsymbol{\theta}) =
J(\boldsymbol{\theta}) +
\underbrace{ \nabla_{\boldsymbol{\theta}}J(\boldsymbol{\theta})^\top \Delta\boldsymbol{\theta} + \tfrac12\,\Delta\boldsymbol{\theta}^\top H(\boldsymbol{\theta})\,\Delta\boldsymbol{\theta} }_{\displaystyle m_{H}(\Delta\boldsymbol{\theta})}
+ \mathcal{O}(\norm{\Delta\boldsymbol{\theta}}^{3}),
\end{equation}
where $H(\boldsymbol{\theta})$ denotes the Hessian of the objective. Equation \ref{eq:taylor} holds under a Lipschitz continuous Hessian (see Assumption \ref{assump:LipH}), with a detailed proof in Appendix \ref{app:hessian}. As the cubic term may be negative, we can establish a guaranteed lower bound $J(\boldsymbol{\theta} + \Delta\boldsymbol{\theta})  \geq J(\boldsymbol{\theta}) + m_{H}(\Delta\boldsymbol{\theta}) - \mathcal{O}(\norm{\Delta\boldsymbol{\theta}}^{3})$. In practice, the cubic term is often negligible, and we approximate the objective change by $m_{H}(\Delta\boldsymbol{\theta})$. Crucially, standard gradient ascent ignores the Hessian contribution, which can lead to a decrease in the objective for non-convex problems (such as RL) when this contribution is sufficiently negative.

\textbf{The Fisher Information Matrix}. The Hessian captures the local curvature of the objective function. In RL, however, the objective is non-stationary, and what ultimately matters is how updates change the policy distribution. For instance, an update may produce only a small change in the objective while inducing a large shift in the policy. This alters how future trajectories are sampled and may destabilize learning. Therefore, it is necessary to track the geometry of the policy distribution directly, which is what the Fisher Information Matrix (FIM) enables. One can show that the directional curvature under the Fisher geometry approximates the average KL divergence between a policy and before and after a small step $\Delta\boldsymbol{\theta}$:
\begin{equation}
    \bar D_{\mathrm{KL}}(\pi_{\boldsymbol{\theta}} \,\|\, \pi_{\boldsymbol{\theta}+\Delta\boldsymbol{\theta}})
= \underbrace{\tfrac12 \Delta\boldsymbol{\theta}^\top F(\boldsymbol{\theta}) \Delta\boldsymbol{\theta}}_{m_{F}(\Delta\boldsymbol{\theta})}
+ \mathcal{O}(\norm{\Delta\boldsymbol{\theta}}^3),
\end{equation}
where $\bar D_{\mathrm{KL}}(\pi_{\boldsymbol{\theta}} \,\|\, \pi_{\boldsymbol{\theta}+\Delta\boldsymbol{\theta}})
:= \mathbb{E}_{s\sim d_\pi}\!\left[
\mathrm{KL}\big(\pi_{\boldsymbol{\theta}}(\cdot\mid s)\,\|\,\pi_{\boldsymbol{\theta}+\Delta\boldsymbol{\theta}}(\cdot\mid s)\big)
\right]$, and $F(\boldsymbol{\theta})
:= \mathbb{E}_{s\sim d_\pi,\ a\sim \pi_{\boldsymbol{\theta}}(\cdot\mid s)}
\Big[ \nabla_{\boldsymbol{\theta}} \log \pi_{\boldsymbol{\theta}}(a\mid s) \,
\nabla_{\boldsymbol{\theta}} \log \pi_{\boldsymbol{\theta}}(a\mid s)^\top \Big]$ is the FIM. The proof is in Appendix \ref{app:fisher}. Similarly to the Hessian case, the cubic term is often negligible and we focus on $m_{F}(\Delta\boldsymbol{\theta})$. One can further show that enforcing a trust region $\bar D_{\mathrm{KL}}(\pi_{\boldsymbol{\theta}} \,\|\, \pi_{\boldsymbol{\theta}+\Delta\boldsymbol{\theta}}) \leq \delta$ during policy updates leads to monotonic improvement of the true objective, given sufficiently small $\delta$ \citep{pmlr-v37-schulman15}.

Ultimately, we aim to design a model that approximates $m_{H}(\Delta\boldsymbol{\theta})$ and $m_{F}(\Delta\boldsymbol{\theta})$ without explicitly computing gradients or curvature terms in the high-dimensional parameter space of the LLM. This approach can be viewed as a form of model-based RL, but from a different perspective: whereas prior work typically models components of the MDP, such as the dynamics or reward function, we instead model the optimization process itself, which allows us to plan gradient estimates.

\subsection{Computational Model}\label{sec:compmodel}
For an LLM with $d$ parameters, both Hessian and FIM are $d \times d$ matrices, which is intractable for billion-size parameter spaces. Even approximations such as K-FAC \citep{NEURIPS2023_6a6679e3} would incur unfeasible memory cost. Therefore, we need to devise a computational model that is scalable and effectively provides curvature information to stabilize policy gradients. Next, we describe our methodology.

\textbf{Last-Layer Model}. Since modeling the full Hessian or Fisher Information Matrix (FIM) is infeasible, we restrict attention to curvature in a parameter subspace. To this end, we adopt a simple last-layer approach. An LLM is a softmax policy over the token vocabulary $ \pi_{\boldsymbol{\theta}}(a \mid s) = \tfrac{\exp(f_\theta(s,a))}{\sum_{a'} \exp(f_\theta(s,a'))}$, where $f_\theta(s, a) \in \mathbb{R}$ are the logits produced by the network. Letting $f_{\boldsymbol{\theta}}(s_{t})$ denote the full logits vector, with $\boldsymbol{\theta} = (\bar{\boldsymbol{\theta}}, \boldsymbol{\psi})$, we represent the pre-softmax layer as $f_{\boldsymbol{\theta}}(s_{t}) = Wh_{\bar{\boldsymbol{\theta}}}(s_{t})$, where $W \in \mathbb{R}^{K\times d_{i}}$ is the last-layer weight matrix, $K = \text{dim}(\mathcal{V})$, and $h_{\bar{\boldsymbol{\theta}}}(s_{t}) \in \mathbb{R}^{d_{i}}$. We then define $\boldsymbol{\psi} = \text{vec}(W) \in \mathbb{R}^{K \cdot d_{i}}$. In Appendix \ref{app:model}, we show that the last-layer model gradient $\Tilde{g}(\boldsymbol{\psi})$ of the objective in Equation \ref{eq:obj} is:
\begin{equation}\label{eq:gradient}
\tilde{g}(\boldsymbol{\psi}) = \mathbb{E}_{\tau \sim \pi_{\boldsymbol{\theta}}} \left[ \sum_{t=0}^T \gamma^{t} A(s_t, a_t) (e_a - \pi_{\boldsymbol{\theta}}(s_{t})) \otimes h_{\bar{\boldsymbol{\theta}}}(s_{t})\right],
\end{equation}
where $\otimes$ denotes a Kronecker product, $e_{a_{t}} \in \mathbb{\mathcal{V}}$ is the one-hot action vector $e_{a_{t}} = \mathbf{1}\{a = a_{t}\}$, and $\pi_{\boldsymbol{\theta}}(s_{t}))$ the policy distribution vector. We use the vectorization operation $\text{vec}(\cdot)$ only for convenience and it does not introduce new assumptions. In this work, we use a tilde superscript to denote \textit{model-based} gradients and curvatures, in contrast to the actual \textit{policy} gradient $g(\boldsymbol{\theta}): = \nabla_{\boldsymbol{\theta}}J(\boldsymbol{\theta})$.

Under the last-layer model, the Hessian of the objective takes the following form: 
\begin{equation}\label{eq:hessian}
   \tilde{H}(\boldsymbol{\psi}) 
= \mathbb{E}_{\tau \sim \pi_{\boldsymbol{\theta}}} \!\left[ \sum_{t=0}^{T}\gamma^{t}
A(s,a) \Big( (e_a - \pi_{\boldsymbol{\theta}} (s_{t}))(e_a - \pi_{\boldsymbol{\theta}} (s_{t}))^\top - F(s_{t}) \Big) \otimes h_{\bar{\boldsymbol{\theta}}}(s_{t}) h_{\bar{\boldsymbol{\theta}}}(s_{t})^\top
\right],
\end{equation}
where $F(s_{t})$ is the FIM for state $s_{t}$. In Lemma \ref{lemma:hessianpg}, we show that this expression can be estimated via samples. Similarly, the last-layer approximation of the FIM is:
\begin{equation}\label{eq:fisher}
   \tilde{F}(\boldsymbol{\psi})
   = \mathbb{E}_{\tau \sim \pi_{\boldsymbol{\theta}}} \!\left[ 
   \big( (e_{a_t} - \pi_{\boldsymbol{\theta}}(s_{t})) (e_{a_t} - \pi_{\boldsymbol{\theta}}(s_{t}))^\top \big) 
   \otimes h_{\bar{\boldsymbol{\theta}}}(s_{t}) h_{\bar{\boldsymbol{\theta}}}(s_{t})^\top
   \right].
\end{equation}
\textbf{Computing \textit{Directional} Curvatures.} Even with the approximated model, the curvature matrices have dimension $Kd_i \times Kd_i$. For current LLMs, where $K > 10^5$ and $d_i > 10^3$, fully materializing these matrices is computationally infeasible. Fortunately, our goal is to approximate the shifts in the objective and policy, $m_H(\Delta\boldsymbol{\theta})$ and $m_F(\Delta\boldsymbol{\theta})$. Thus, we only need to approximate the \textit{directional} curvatures $\Delta\boldsymbol{\theta}^\top C(\boldsymbol{\theta})\,\Delta\boldsymbol{\theta}$, without explicitly materializing the full Hessian or FIM. In Appendix \ref{app:compdir}, we present a mechanism that enables this computation without constructing large tensors. Our method requires storing only $\mathcal{O}(Kd_i)$ tensors per state--action sample, instead of the $\mathcal{O}((Kd_i)^2)$ entries of the full curvature matrices.

\textbf{Exploiting Gradient Sparsity}. We further reduce complexity by exploiting the structure of gradients arising from LLM generation. Standard LLM decoding relies on selective sampling methods (e.g., top-k, nucleus sampling) \cite{wolf2020huggingfacestransformersstateoftheartnatural} to improve generation quality, as most of the probability mass is concentrated on a small subset $k$ of the vocabulary \citep{fan-etal-2018-hierarchical, Holtzman2020The}, typically with $k < 100$. Consequently, only $k$ tokens have non-zero probability at each generation step, which implies that only the $k * d_{i}$ parameters of the last-layer weight matrix $W$ associated with these logits yield non-zero gradients. We therefore store and operate these gradients in sparse form. This sparsity also applies to the computation of directional curvatures in Equations \ref{eq:fishersamples} and \ref{eq:hessiansamples}, as these reduce to dot products involving sparse vectors (e.g., $(e_{a_t} - \pi_{\boldsymbol{\theta}}(s_{t})$ and the model-based update step $\Delta{\boldsymbol{\theta}}$). Naturally, as we estimate gradients with more samples, the representation expands to cover all $\tilde{k}$ tokens generated, but typically $\tilde{k} << K$. For instance, our experiments presented $\tilde{k} < 10^4$. Overall, the memory and dot product complexity reduce to $\mathcal{O}(\tilde{k} \cdot d_{i})$.

\textbf{Modeling the Step $\Delta\boldsymbol{\theta}$.} A final design choice concerns how to model the planned update steps, $\Delta\boldsymbol{\theta}$. Under the last-layer model, these steps take the form $\Delta\boldsymbol{\psi}$. This choice essentially determines how we represent the optimizer. A simple option is to model the update as a stochastic gradient descent (SGD) step, $\Delta \boldsymbol{\psi} = \alpha \tilde{g}$, where $\alpha$ is the learning rate. Alternatively, we can match the LLM optimizer, which in our case is Adam \citep{kingma2015adam},  i.e., $\Delta\boldsymbol{\psi} = \alpha \tfrac{\hat{p}_{t}}{\sqrt{\hat{q}_{t}} + \epsilon}$, where $\hat{p}_{t}$ and $\hat{q}_{t}$ are the bias-corrected first and second moment estimates of the gradient.

\section{Curvature-Aware Policy Optimization}\label{sec:capo}

We may now compute the objective and policy shifts under our model as:
\begin{equation}\label{eq:shifts}
    m_{H}(\boldsymbol{\psi}) = \tilde{g}(\boldsymbol{\psi})^\top \Delta \boldsymbol{\psi} + \tfrac{1}{2}\Delta \boldsymbol{\psi}^\top \tilde{H}(\boldsymbol{\psi})\Delta \boldsymbol{\psi}, \quad m_{F}(\boldsymbol{\psi}) = \tfrac{1}{2}\Delta \boldsymbol{\psi}^\top \tilde{F}(\boldsymbol{\psi})\Delta \boldsymbol{\psi}, 
\end{equation}
and estimate $m_{H}$ and $m_{F}$ via samples following the methodology described in the subsection \ref{sec:compmodel}. We now design an algorithm that intervenes in the optimization of the underlying LLM policy using the model-based updates. Since our objective is to stabilize policy gradients in sample-efficient regimes, a natural choice is to construct an algorithm that follows the principles of trust-region methods \citep{pml1Book}. We implement this idea through a rejection sampling mechanism. 

Given a batch $\mathcal{B}$ of collected trajectories, 
we partition it into disjoint subsets $b_{i} \subset \mathcal{B}$. For each subset, we compute a proposed step $\Delta\boldsymbol{\psi}_{i}$ and evaluate the shifts defined in Equation \ref{eq:shifts}. We then accept a subset if it satisfies the (local) trust-region constraints $\delta_{F}$, $\delta_{H}$, and $\delta_{H}^{high}$:
\begin{equation}
    \delta_{H} \le m_H(\Delta\psi_i)\;\le\; \delta_{H}^{high},\qquad 
m_F(\Delta\psi_i)\;\le\;\delta_F.
\end{equation}
The accepted subsets are subsequently used to compute the gradient update of the LLM policy. Conceptually, this mechanism is analogous to token masking. Overall, this data selection mechanism is simple, computationally inexpensive, and flexible, as it can be applied at different granularities, including tokens, sentences, groups, or full batches. The formal pseudocode is provided in Algorithm \ref{alg:capo}. Next, we establish theoretical results for monotonic policy improvement under CAPO.

\begin{theorem}[Monotonic improvement under CAPO]\label{thm:capo-certified-main}
Fix thresholds $\delta_H>0$ and $\delta_F>0$. Let $\mathcal{B}$ be a batch of sampled trajectories. Split $\mathcal{B}$ into disjoint $N$ subsets  $b_{i} \subset \mathcal{B}$, and propose candidate subset updates $\{\Delta\theta_i\}_{i:N}$. Retain those satisfying:
\begin{equation}
    m_H(\Delta\theta_i)\;\ge\; \delta_{H} = \omega+\tfrac12Mr^2,\qquad 
m_F(\Delta\theta_i)\;\le\;\delta_F,
\end{equation}
with $\omega >  0$ and $M$, $r$ defined as in Assumption \ref{ass:curvature-and-steps}. Let $\mathcal{B}_{acc}$ denote the superset of the B accepted subsets, and define the aggregated update:$    \Delta\theta \;=\; \frac{1}{B}\sum_{i\in\mathcal{B}_{acc}}\Delta\theta_i$.
Then, for two policies $\pi_{\boldsymbol{\theta}}$ and $\pi_{\boldsymbol{\theta} + \Delta \boldsymbol{\theta}}$, with $|A^\pi(s,a)|\le \epsilon$, we obtain:
\begin{equation}
    J(\pi_{\theta+\Delta\theta}) - J(\pi_\theta) 
\;\ge\; \omega \;-\; C\,\sqrt{\delta_F}, \quad C=\frac{2\gamma}{(1-\gamma)^2}\,\epsilon\,\sqrt{2}. 
\end{equation} 
Thus choosing $\omega \ge C\sqrt{\delta_F}$ guarantees monotonic improvement: $J(\pi_{\theta+\Delta\theta}) \ge J(\pi_\theta)$.
\end{theorem}
The proof is provided in Appendix \ref{app:capo}. Observe that $\delta_{H}^{high}$ is not required to establish monotonic improvement. Nonetheless, it serves as a safeguard against overly aggressive steps. In practice, introducing this upper cap reduces the observed $M$ and $r$, which allows the use of smaller $\delta_{H}$. Finally, we note that Theorem \ref{thm:capo-certified-main} relies on the true objective and policy shifts, whereas in practice these quantities are approximated using our model.
\section{Experiments and Discussion}\label{sec:results}
In this section, we evaluate (i) how the proposed computational model captures the optimization landscape, and (ii) how this information can be used to stabilize RL optimization dynamics through CAPO. Our central hypothesis is that an inexpensive yet effective approximation of second-order geometry can track unstable shifts in the objective and policy, and that this information can in turn be used to stabilize aggressive update regimes, leading to more sample-efficient RL in LLMs. 
\setlength{\textfloatsep}{0pt}
\setlength{\intextsep}{0pt}%

\textbf{Experimental Setup.} We consider a standard RL setup for finetuning LLMs on reasoning tasks. Our implementation builds on the Open-R1 open-source project \citep{openr1}, and we maximize an accuracy-based reward. Following prior work, we fine-tune a Qwen2.5-Math-7B LLM \citep{qwen2025qwen25technicalreport} on mathematical reasoning questions. Our primary evaluation metric is accuracy, but we also track optimization-related quantities such as gradient and curvature statistics and token rejection rates. Since our goal is to evaluate sample efficiency, we report all metrics as a function of the number of training completions (i.e., LLM \textit{trajectories} generated). Appendix \ref{app:implementation} provides additional details regarding implementation, hyperparameters, and compute resources\footnote{We release our code at \codeurl.}.

\textbf{Datasets \& Benchmarks.}  We train our policies on the MATH dataset \citep{hendrycksmath2021}. For evaluation, we consider eight benchmarks: GSM8K \citep{cobbe2021gsm8k}, MATH500 \citep{lightman2023lets}, OlympiadBench \citep{he2024olympiadbench}, MinervaMath \citep{lewkowycz2022solvingquantitativereasoningproblems}, GPQA:Diamond \citep{rein2023gpqagraduatelevelgoogleproofqa}, AMC23, AIME24, and AIME25. Most of these benchmarks contain mathematical questions at varying levels (high school, graduate, and olympiad), while GPQA focuses on general STEM-related problems. For simplicity, we report the average performance across all eight benchmarks, which we refer to as “TEST” in the results.

\textbf{Comparison Methods}. We evaluate our approach against two GRPO variants. The first corresponds to the standard “conservative” update regime implemented in the Open-R1 codebase. The second, which we denote “GRPO (A),” adopts a more aggressive regime intended to improve sample efficiency, with a learning rate $5\times$ higher and a batch size $12\times$ smaller. This matches the configuration used by CAPO. We also evaluate Dr.GRPO \citep{liu2025understanding} and REINFORCE \citep{Williams:92}, both under the same aggressive regime.

\textbf{CAPO operationalization.} CAPO optimizes the same objective as GRPO, but leverages the data selection mechanism introduced in Section~\ref{sec:capo}. For a fair comparison, we use the same hyperparameters as GRPO (A). We implement CAPO with token-level selection, i.e., proposing steps $\Delta \boldsymbol{\psi}_{i}$ and rejecting samples on a per-token basis. Finally, we model optimization steps using Adam.

 \begin{figure}[t]
 \setlength{\abovecaptionskip}{0pt} 
\begin{center}
\includegraphics[width=1.0\textwidth]{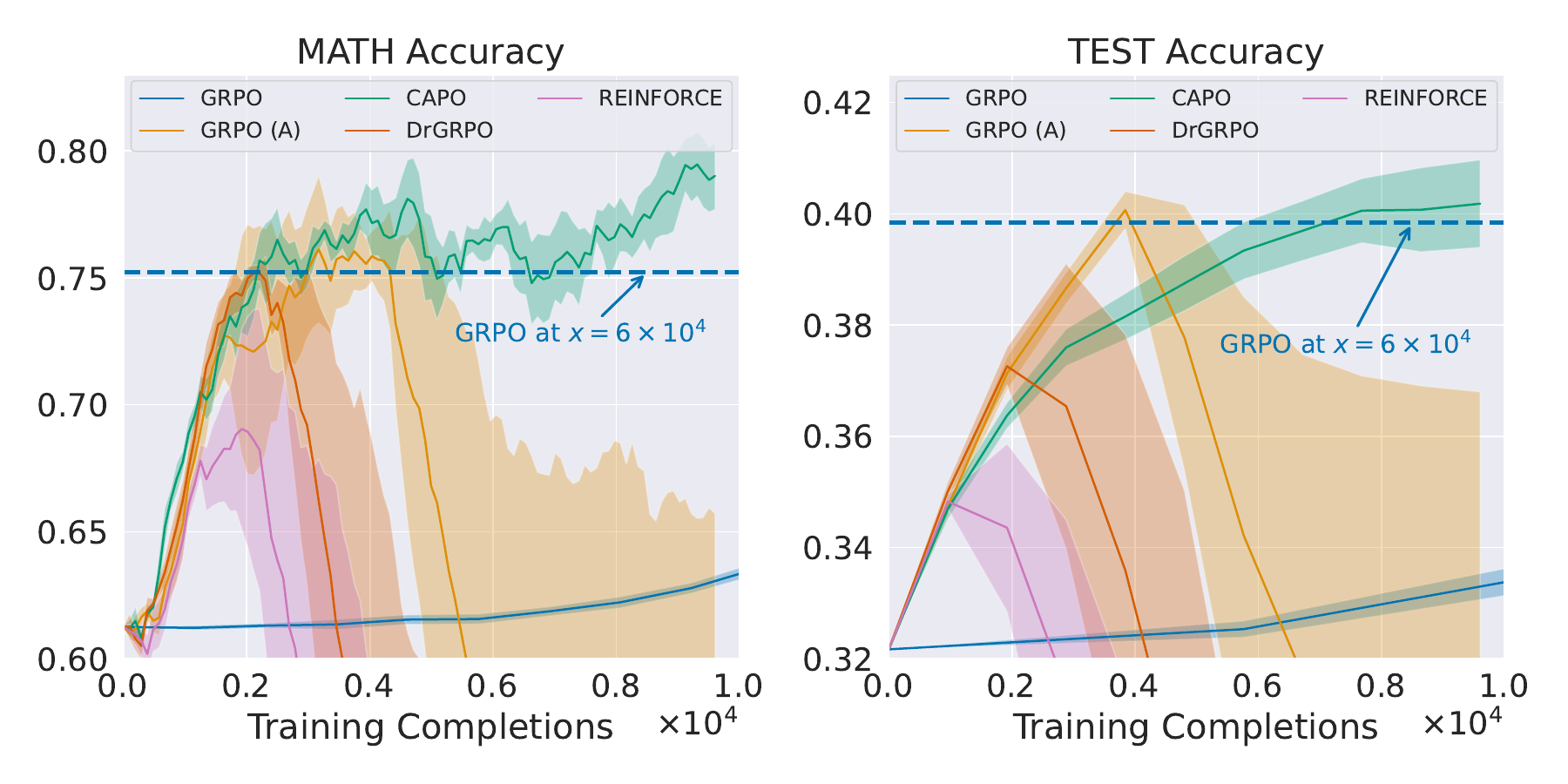}
\caption{\textbf{Comparison with baseline methods on policy gradient stability.} 
While the setup with more aggressive updates makes all methods more sample-efficient, it also leads the baselines to policy collapse. In contrast, CAPO prevents collapse and achieves up to $30\times$ greater sample efficiency 
than GRPO under aggressive updates.}
\label{fig:eval_baselines}
\end{center}
\end{figure}

\subsection{Experiments}

 We highlight and analyze the following questions to evaluate our hypothesis and proposed method:

 \textbf{Does CAPO prevent instability in LLM policy gradients? Does it lead to better sample efficiency?} Figure \ref{fig:eval_baselines} reports accuracy for all methods on MATH and on the TEST benchmark set. First, we observe that the more aggressive setup does lead to more sample-efficient learning than the conservative one across all methods. However, for the baselines, this improvement comes at the cost of stability. Under the aggressive regime, all baseline methods suffer from policy collapse, with performance dropping well below that of the base model and therefore losing the ability to learn further. In contrast, CAPO maintains stable performance throughout training, remaining effective long after all other methods have collapsed. This demonstrates that CAPO effectively prevents instability under aggressive updates. As a result, CAPO requires $30\times$ fewer completions on MATH and $9\times$ fewer completions on TEST compared to standard conservative GRPO.

\begin{figure}[t]
\begin{center}
\includegraphics[width=1.0\textwidth]{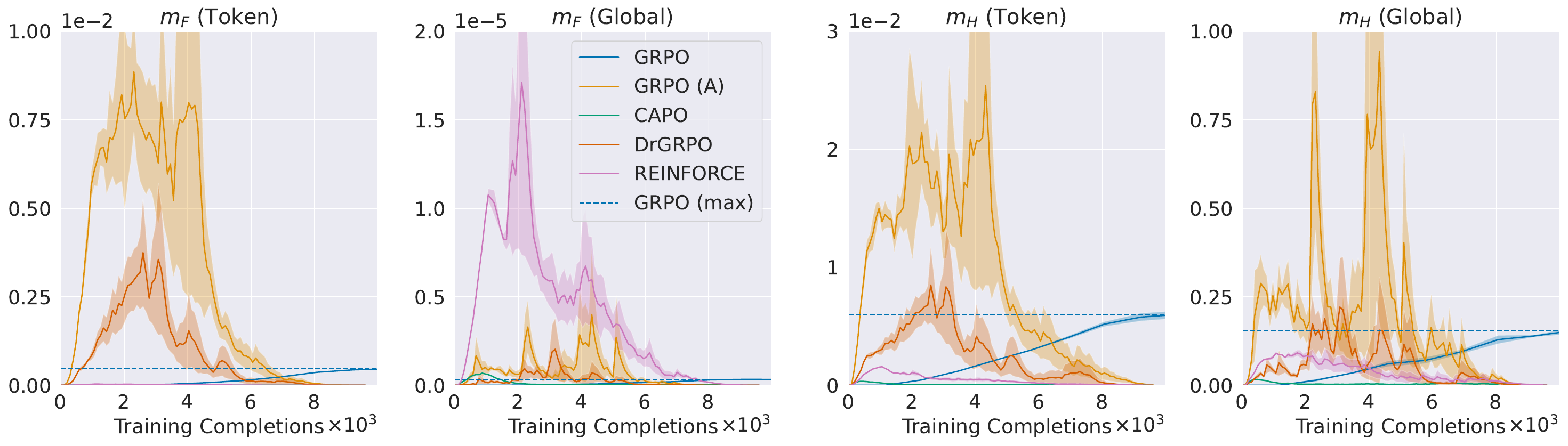}
\caption{\textbf{Evaluation of policy and objective shifts estimates from the proposed computational model during training.} 
Unstable methods exhibit large and abrupt directional curvatures, while stable ones maintain much smaller and smoother shifts. CAPO, by applying token-level bounds, also ensures well-behaved shifts at the global (batch) level, supporting the rationale of Theorem~\ref{thm:capo-certified-main}.}
\label{fig:shifts}
\end{center}
\end{figure}

\textbf{What does the proposed computational model reveal about the optimization landscape?} 
To analyze this question, we examine the policy shift $m_F$ and the objective shift $m_H$ at both the token level and the global (batch) level over the course of training, presented in Figure \ref{fig:shifts}. For $m_F$, we find that unstable methods (GRPO (A), DrGRPO, REINFORCE) exhibit very high global directional curvatures during training, whereas stable methods (CAPO, standard GRPO) maintain much smaller shifts. In particular, the global $m_F$ correlates closely with the instability observed in Figure~\ref{fig:math_accuracy}, showing that the model, despite its simplicity, remains informative about optimization dynamics.  

For $m_H$, we observe similar trends: unstable methods show abrupt shifts, while stable ones produce smoother, better-behaved curves. Note that, while a higher $m_F$ directly signals instability since it tracks policy shifts, a higher $m_H$ does not necessarily directly imply instability. This is because $m_H$ depends on the adopted advantage function (Equation~\ref{eq:hessianadv}) and the normalization strategy of each method. Still, sharp peaks in the $m_H$ curves also correlate with training instabilities.  Lastly, we highlight that CAPO, by applying a local bound per token, also ensures well-behaved shifts at the global level, which supports the rationale of Theorem \ref{thm:capo-certified-main}. Overall, these results highlight that the computational model provides meaningful information about the optimization landscape, and that CAPO effectively leverages this information to stabilize training.

\begin{figure}[t]
\begin{center}
\includegraphics[width=1.0\textwidth]{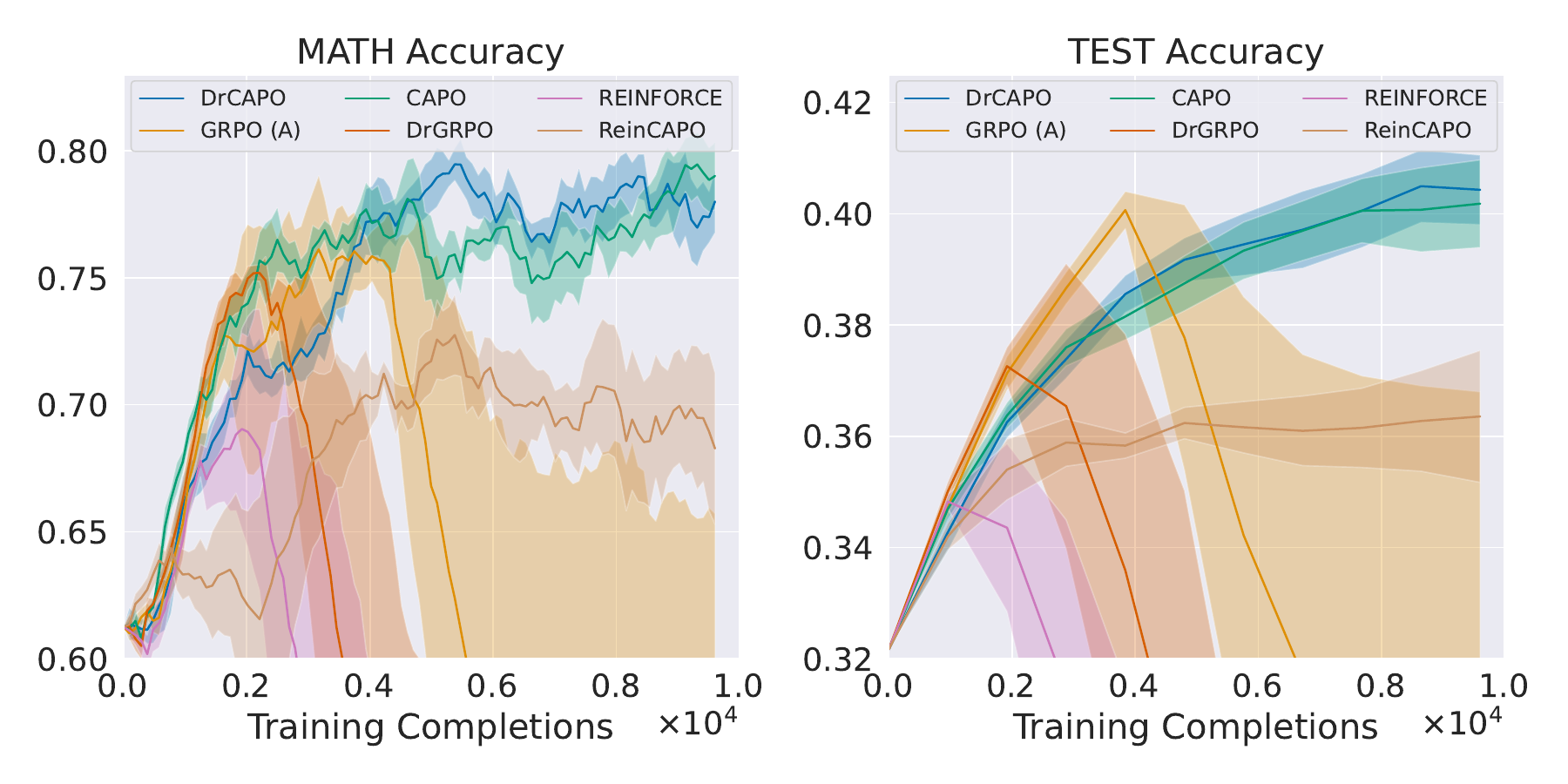}
\caption{\textbf{Evaluation of extended versions of RL methods with curvature-aware selection.} Incorporating curvature-aware selection consistently improves the base methods, preventing policy collapse and demonstrating the broader applicability of our approach across different policy optimization objectives.}
\label{fig:capo_extensions}
\end{center}
\end{figure}

 \begin{wrapfigure}{r}{0.4\textwidth}
\begin{center}
\includegraphics[width=0.4\textwidth]{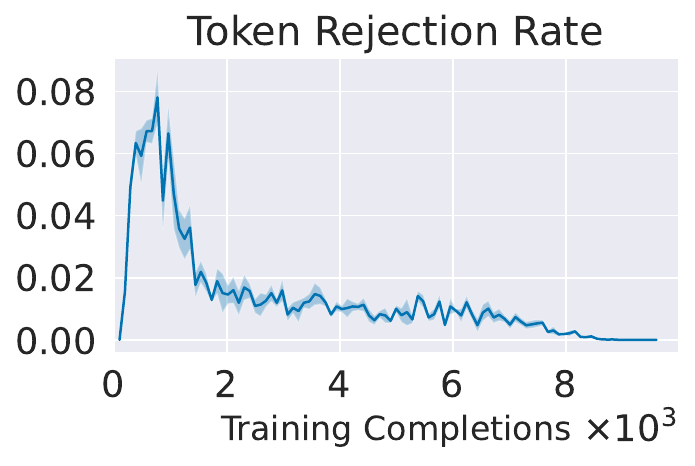}
\captionsetup{skip=-1pt}
\caption{\textbf{Token rejection rate under CAPO.} It maintains a low rejection rate over training, stabilizing learning with minimal intervention.}
\label{fig:clip_ratio}
\end{center}
\end{wrapfigure}

\textbf{Can we extend curvature-aware selection to other RL methods?} To test this, we extend Dr.GRPO and REINFORCE by incorporating our proposed curvature-aware selection, resulting in Dr.CAPO and ReinCAPO, respectively. Figure \ref{fig:capo_extensions} reports the evaluation results for these methods. In all cases, incorporating the selection strategy improves upon the base method and prevents policy collapse. These findings suggest that the proposed computational model and intervention mechanism are broadly applicable across different policy optimization objectives.

\textbf{How aggressive is CAPO’s intervention to ensure stability?} 
We analyze the extent of token rejection required by CAPO to maintain stable gradients, measured by the token rejection rate during training (Figure \ref{fig:clip_ratio}). 
The rejection rate peaks at about $8\%$ in the early stages of optimization, when higher learning rates produce more aggressive updates, but quickly decreases and remains below $2\%$ for the remainder of training. Overall, this shows that CAPO guides optimization toward stable curvature regions while keeping its intervention minimal, allowing the LLM to continue leveraging the vast majority of samples.

\textbf{Additional Experiments}. We provide a computational cost analysis of CAPO in Appendix \ref{app:compcost}, where we show that the additional components incur minor overhead. Additionally, we present further experiments in Appendix \ref{app:ablations}, including an ablation study on the optimizer model and a detailed evaluation of other heuristics traditionally used to ensure stability (e.g., PPO clipping and KL regularization), highlighting their limitations in the LLM setup.

\section{Final Remarks}
In this work, we propose a computational framework that models curvature information and integrates it into policy updates through CAPO. We provide theoretical guarantees for CAPO and show that it is effective at identifying samples that contribute to unstable updates, preventing policy collapse in aggressive training regimes where standard RL methods for LLM reasoning fail. As a result, CAPO achieves up to a $30\times$ improvement in sample efficiency compared to widely used training setups, while requiring only minimal intervention and computational overhead. Overall, it enables more sample-efficient learning regimes, supporting further scalability post-training scalability.

\textbf{Limitations.} Despite the encouraging results, we acknowledge some limitations of our work. First, due to compute budget constraints, we focused on experiments at a smaller, academic scale. While we demonstrated the effectiveness of CAPO against commonly used RL methods, future work could extend these results to distinct problem settings and longer training schedules. Second, the choice of CAPO thresholds depends on the problem setting (MDP, objective function, base policy) and may require tuning across different scenarios. Nonetheless, this is not a major concern, as the thresholds can be tuned solely on the training distribution.

\textbf{Future Work.} Beyond scalability, future research may explore different parametrizations of the computational model (for instance, by extending it to deeper layers) and investigate their impact on computational tractability and curvature estimates. In addition, future work may evaluate CAPO extensions to other intervention mechanisms, such as soft masking or regularization methods.

\section*{Acknowledgments}
We thank Shreshth Malik, Clare Lyle, and Yoav Gelberg for the insightful discussions in the early stages of this project. Luckeciano C. Melo acknowledges funding from the Air Force Office of Scientific Research (AFOSR) European Office of Aerospace Research \& Development (EOARD) under grant number FA8655-21-1-7017. Yarin Gal is supported by a Turing AI Fellowship financed by the UK government’s Office for Artificial Intelligence, through UK Research and Innovation (grant reference EP/V030302/1) and delivered by the Alan Turing Institute. We gratefully acknowledge compute resources provided by the Alan Turing Institute. This work was also funded under the Horizon Europe grant 101213369 DVPS.

\bibliography{iclr2026_conference}
\bibliographystyle{iclr2026_conference}

\appendix
\clearpage

\section{Derivation of the Second-Order Optimization Objective}\label{app:hessian}

In this section, we formally derive the higher-order expansion of the objective function around a given parameter vector, and present conditions for monotonic improvement. We start by highlighting a smoothness assumption required for our analysis.

\begin{assumption}[Lipschitz continuity of the Hessian]\label{assump:LipH}
There exists a constant \(L_2 \ge 0\) such that, for all \(\tau \in [0,1]\) and all \(\Delta\boldsymbol{\theta} \in \mathbb{R}^d\),
\begin{equation}
    \big\| \nabla^2 J(\boldsymbol{\theta}+\tau \Delta\boldsymbol{\theta}) - \nabla^2 J(\boldsymbol{\theta}) \big\|_{\op}
\;\le\; L_2 \,\tau \,\norm{\Delta\boldsymbol{\theta}}.
\end{equation}
\end{assumption}

Assumption~\ref{assump:LipH} is standard in the analysis of trust-region and cubic-regularized methods, and holds locally for smooth policy parameterizations. 

\begin{proposition}[Second-order expansion with integral remainder]\label{prop:deterministic}
Let \(J:\mathbb{R}^d \to \mathbb{R}\) be three times differentiable, and denote $g \;\triangleq\; \nabla J(\boldsymbol{\theta})$ and $H \;\triangleq\; \nabla^2 J(\boldsymbol{\theta})$.
For any update direction \(\Delta\boldsymbol{\theta} \in \mathbb{R}^d\), the objective value at the perturbed parameter \(\boldsymbol{\theta}+\Delta\boldsymbol{\theta}\) admits the expansion
\begin{equation}
J(\boldsymbol{\theta}+\Delta\boldsymbol{\theta})
=
J(\boldsymbol{\theta})
+ g^\top \Delta\boldsymbol{\theta} 
+ \tfrac12\,\Delta\boldsymbol{\theta}^\top H\,\Delta\boldsymbol{\theta}
+ \int_0^1 (1-\tau)\,\Delta\boldsymbol{\theta}^\top\!\big(\nabla^2 J(\boldsymbol{\theta}+\tau\Delta\boldsymbol{\theta})-H\big)\Delta\boldsymbol{\theta} \, d\tau.
\label{eq:integral-remainder}
\end{equation}
Under Assumption \ref{assump:LipH}, the following lower-bound holds

\begin{equation}\label{eq:hessianbound}
    J(\boldsymbol{\theta}+\Delta\boldsymbol{\theta}) \;\ge\; J(\boldsymbol{\theta}) + g^\top \Delta\boldsymbol{\theta} + \tfrac12\,\Delta\boldsymbol{\theta}^\top H\,\Delta\boldsymbol{\theta} - \tfrac{L_2}{6}\,\norm{\Delta\boldsymbol{\theta}}^3.
\end{equation}
\end{proposition}

\begin{proof}
Let \(\phi(\tau)=J(\boldsymbol{\theta}+\tau\Delta\boldsymbol{\theta})\) for \(\tau\in[0,1]\).
Then \(\phi'(0)=g^\top\Delta\boldsymbol{\theta}\) and \(\phi''(0)=\Delta\boldsymbol{\theta}^\top H \Delta\boldsymbol{\theta}\).
The (one-dimensional) Taylor formula with integral remainder gives
\begin{equation}
    \phi(1)=\phi(0)+\phi'(0)+\tfrac12\phi''(0)
+ \int_0^1 (1-\tau)\left(\phi''(\tau)-\phi''(0)\right)\,d\tau .
\end{equation}

Since \(\phi''(\tau)-\phi''(0) = \Delta\boldsymbol{\theta}^\top(\nabla^2J(\boldsymbol{\theta}+\tau\Delta\boldsymbol{\theta})-H)\Delta\boldsymbol{\theta}\), we obtain \eqref{eq:integral-remainder}.
Assumption~\ref{assump:LipH} implies
\(\big|\Delta\boldsymbol{\theta}^\top(\nabla^2J(\boldsymbol{\theta}+\tau\Delta\boldsymbol{\theta})-H)\Delta\boldsymbol{\theta}\big|
\le \norm{\Delta\boldsymbol{\theta}}^2 \,\big\|\nabla^2J(\boldsymbol{\theta}+\tau\Delta\boldsymbol{\theta})-H\big\|_{\op}
\le L_2 \tau \norm{\Delta\boldsymbol{\theta}}^3\).
Solving the integral gives \(\int_0^1 (1-\tau)\,L_2\tau \,d\tau = L_2/6\). Since this term can be negative, a worst-case bound yields the inequality \ref{eq:hessianbound}.
\end{proof}

\clearpage

\section{Derivation of the Policy Divergence Quadratic Approximation}\label{app:fisher}

In this section, we formally derive the higher-order expansion of the KL term around a small step 
$\Delta\boldsymbol{\theta}$. Throughout this derivation, we assume standard regularity assumptions 
hold (e.g., parameter-independent support, differentiability of $\log \pi_{\boldsymbol{\theta}}$, 
and dominated convergence so that differentiation may pass under the expectation). The state averaging distribution $d_\pi$ is fixed.

\begin{assumption}[Lipschitz continuity of the Fisher curvature]\label{assump:fisher_lip}
Let $F(\boldsymbol{\theta})
:= \mathbb{E}_{s\sim d_\pi,\ a\sim \pi_{\boldsymbol{\theta}}(\cdot\mid s)}
\Big[ \nabla_{\boldsymbol{\theta}} \log \pi_{\boldsymbol{\theta}}(a\mid s) \,
\nabla_{\boldsymbol{\theta}} \log \pi_{\boldsymbol{\theta}}(a\mid s)^\top \Big].$

There exists a constant $L_F \ge 0$ such that, for all $\tau \in [0,1]$ and all 
$\Delta\boldsymbol{\theta}\in\mathbb{R}^d$,
\begin{equation}
\big\|F(\boldsymbol{\theta}+\tau\Delta\boldsymbol{\theta}) - F(\boldsymbol{\theta})\big\|_{\mathrm{op}}
\;\le\; L_F \,\tau \,\|\Delta\boldsymbol{\theta}\|.
\end{equation}

\end{assumption}

Assumption \ref{assump:fisher_lip} is analogous to the Assumption \ref{assump:LipH} applied to the Fisher geometry.

\begin{lemma}[The grad-log-prob identity]\label{lem:score_zero} Under regularity assumptions, the following identity holds:
\begin{equation}
    \mathbb{E}_{s\sim d_\pi,\ a\sim \pi_{\boldsymbol{\theta}}(\cdot\mid s)}
\big[\nabla_{\boldsymbol{\theta}} \log \pi_{\boldsymbol{\theta}}(a\mid s)\big] = 0.
\end{equation}
\end{lemma}

\begin{proof}
Fix $s$. By normalization, $\sum_a \pi_{\boldsymbol{\theta}}(a\mid s)=1$. Differentiating, $\sum_a \nabla_{\boldsymbol{\theta}} \pi_{\boldsymbol{\theta}}(a\mid s) = 0$. Since $\nabla_{\boldsymbol{\theta}} \pi_{\boldsymbol{\theta}} = 
\pi_{\boldsymbol{\theta}} \nabla_{\boldsymbol{\theta}} \log \pi_{\boldsymbol{\theta}}$, we obtain
\begin{equation}
    \sum_a \pi_{\boldsymbol{\theta}}(a\mid s) \, \nabla_{\boldsymbol{\theta}} \log \pi_{\boldsymbol{\theta}}(a\mid s) = 0,
\end{equation}
i.e.
\(\mathbb{E}_{a\sim \pi_{\boldsymbol{\theta}}(\cdot\mid s)}[\nabla_{\boldsymbol{\theta}} \log \pi_{\boldsymbol{\theta}}(a\mid s)] = 0\).
Averaging over $s \sim d_\pi$ preserves zero.
\end{proof}

\begin{lemma}[Fisher identity]\label{lem:fisher_identity} Under regularity assumptions, the following identity holds:
\begin{equation}
    -\mathbb{E}\big[\nabla_{\boldsymbol{\theta}}^2 \log \pi_{\boldsymbol{\theta}}(a\mid s)\big]
= \mathbb{E}\big[\nabla_{\boldsymbol{\theta}} \log \pi_{\boldsymbol{\theta}}(a\mid s)\,
\nabla_{\boldsymbol{\theta}} \log \pi_{\boldsymbol{\theta}}(a\mid s)^\top\big]
=:\,F(\boldsymbol{\theta}).
\end{equation}

\end{lemma}

\begin{proof}
Fix $s$. Twice differentiating normalization gives $\nabla_{\boldsymbol{\theta}}^2 \sum_a \pi_{\boldsymbol{\theta}}(a\mid s)
  = \sum_a \nabla_{\boldsymbol{\theta}}^2 \pi_{\boldsymbol{\theta}}(a\mid s) = 0$. Using 
$\nabla_{\boldsymbol{\theta}}^2 \pi_{\boldsymbol{\theta}}
= \pi_{\boldsymbol{\theta}}\big(\nabla_{\boldsymbol{\theta}}^2 \log \pi_{\boldsymbol{\theta}}
 + \nabla_{\boldsymbol{\theta}} \log \pi_{\boldsymbol{\theta}}\, 
   \nabla_{\boldsymbol{\theta}} \log \pi_{\boldsymbol{\theta}}^\top\big)$,
we obtain
\begin{equation}
    0 = \sum_a \pi_{\boldsymbol{\theta}}(a\mid s)\,\nabla_{\boldsymbol{\theta}}^2 \log \pi_{\boldsymbol{\theta}}(a\mid s)
  + \sum_a \pi_{\boldsymbol{\theta}}(a\mid s)\,
    \nabla_{\boldsymbol{\theta}} \log \pi_{\boldsymbol{\theta}}(a\mid s)\,
    \nabla_{\boldsymbol{\theta}} \log \pi_{\boldsymbol{\theta}}(a\mid s)^\top.
\end{equation}

Recognizing expectations over $a\sim \pi_{\boldsymbol{\theta}}(\cdot\mid s)$ and multiplying by $-1$ yields
\begin{equation}
    -\mathbb{E}_{a\sim\pi_{\boldsymbol{\theta}}(\cdot\mid s)}[\nabla_{\boldsymbol{\theta}}^2 \log \pi_{\boldsymbol{\theta}}(a\mid s)]
= \mathbb{E}_{a\sim\pi_{\boldsymbol{\theta}}(\cdot\mid s)}[\nabla_{\boldsymbol{\theta}} \log \pi_{\boldsymbol{\theta}}(a\mid s)\,
   \nabla_{\boldsymbol{\theta}} \log \pi_{\boldsymbol{\theta}}(a\mid s)^\top].
\end{equation}

Averaging over $s\sim d_\pi$ gives the result.
\end{proof}

\begin{proposition}[Second-order expansion with integral remainder]\label{prop:kl_expansion}
Define the average forward KL as
\begin{equation}
    \bar D_{\mathrm{KL}}(\pi_{\boldsymbol{\theta}} \,\|\, \pi_{\boldsymbol{\theta}+\Delta\boldsymbol{\theta}})
:= \mathbb{E}_{s\sim d_\pi}\!\left[
\mathrm{KL}\big(\pi_{\boldsymbol{\theta}}(\cdot\mid s)\,\|\,\pi_{\boldsymbol{\theta}+\Delta\boldsymbol{\theta}}(\cdot\mid s)\big)
\right].
\end{equation}

Then, for any update $\Delta\boldsymbol{\theta}$,
\begin{equation}
    \bar D_{\mathrm{KL}}(\pi_{\boldsymbol{\theta}} \,\|\, \pi_{\boldsymbol{\theta}+\Delta\boldsymbol{\theta}})
= \tfrac12 \Delta\boldsymbol{\theta}^\top F(\boldsymbol{\theta}) \Delta\boldsymbol{\theta}
+ \int_0^1 (1-\tau)\,\Delta\boldsymbol{\theta}^\top\!\big(F(\boldsymbol{\theta}+\tau\Delta\boldsymbol{\theta})
- F(\boldsymbol{\theta})\big)\Delta\boldsymbol{\theta}\,d\tau.
\end{equation}

And, under Assumption \ref{assump:fisher_lip}, the following holds:
\begin{equation}
    \bar D_{\mathrm{KL}}(\pi_{\boldsymbol{\theta}} \,\|\, \pi_{\boldsymbol{\theta}+\Delta\boldsymbol{\theta}})
= \tfrac12 \Delta\boldsymbol{\theta}^\top F(\boldsymbol{\theta}) \Delta\boldsymbol{\theta}
+ \mathcal{O}(\norm{\Delta\boldsymbol{\theta}}^3).
\end{equation}

\end{proposition}

\begin{proof}
Let $\phi(\tau) := \bar D_{\mathrm{KL}}(\pi_{\boldsymbol{\theta}} \,\|\, \pi_{\boldsymbol{\theta}+\tau\Delta\boldsymbol{\theta}})$. By the Taylor expansion with integral remainder,
\begin{equation}
    \phi(1) = \phi(0) + \phi'(0) + \tfrac12 \phi''(0) 
+ \int_0^1 (1-\tau)\big(\phi''(\tau)-\phi''(0)\big)\,d\tau.
\end{equation}

Then $\phi(0)=0$, and $\phi'(\tau) = -\,\mathbb{E}\!\left[\nabla_{\boldsymbol{\theta}} \log \pi_{\boldsymbol{\theta}+\tau\Delta\boldsymbol{\theta}}(a\mid s)\right]^\top \Delta\boldsymbol{\theta}$,
so by Lemma~\ref{lem:score_zero}, $\phi'(0)=0$. Differentiating again and applying Lemma \ref{lem:fisher_identity},

\begin{equation}
    \phi''(\tau) = \Delta\boldsymbol{\theta}^\top F(\boldsymbol{\theta}+\tau\Delta\boldsymbol{\theta}) \Delta\boldsymbol{\theta},
\qquad
\phi''(0) = \Delta\boldsymbol{\theta}^\top F(\boldsymbol{\theta}) \Delta\boldsymbol{\theta}.
\end{equation}

Substituting the evaluated terms yields the expansion.

Finally, Assumption~\ref{assump:fisher_lip} implies
\begin{equation}
    \big|\Delta\boldsymbol{\theta}^\top\!\big(F(\boldsymbol{\theta}+\tau\Delta\boldsymbol{\theta}) - F(\boldsymbol{\theta})\big)\Delta\boldsymbol{\theta}\big|
\le L_F \tau \|\Delta\boldsymbol{\theta}\|^3.
\end{equation}

Integrating $\int_0^1 (1-\tau)\tau\,d\tau = 1/6$, so the remainder term is $\mathcal{O}(\norm{\Delta\boldsymbol{\theta}}^3)$.
\end{proof}

\clearpage

\section{Derivation of gradients and curvatures under last-layer model}\label{app:model}

In this section, we formally derive the gradient and curvature expressions assuming the last-layer model.

\begin{proposition}[Gradient w.r.t.last-layer model of a softmax policy]\label{prop:grad}
Let us consider a softmax policy $
\pi_{\boldsymbol{\theta}}(a \mid s)
= \frac{\exp\!\big(f_{\boldsymbol{\theta}}(s,a)\big)}{\sum_{a'} \exp\!\big(f_{\boldsymbol{\theta}}(s,a')\big)}$. Let us also denote the pre-softmax layer by $f_{\boldsymbol{\theta}}(s_{t})=Wh_{\bar{\boldsymbol{\theta}}}(s_{t}), W\in\mathbb{R}^{K\times d_{i}},\; h_{\bar{\boldsymbol{\theta}}}(s_{t})\in\mathbb{R}^{d_{i}}$. Define $\boldsymbol{\psi} \coloneqq \mathrm{vec}(W)\in\mathbb{R}^{Kd}$, with $\boldsymbol{\theta} = (\bar{\boldsymbol{\theta}}, \boldsymbol{\psi})$, $K = \text{dim}(\mathcal{V})
$. Then the policy gradient with respect to $\boldsymbol{\psi}$ of the PG objective:
\begin{equation}
    J(\theta) = \mathbb{E}_{\tau \sim \pi_\theta} \!\left[ \sum_{t=0}^T \gamma^t A(s_t,a_t) \log \pi_{\boldsymbol{\theta}}(a_t \mid s_t) \right]
\end{equation}
is given by:
\begin{equation}
    \tilde{g}(\boldsymbol{\psi})
= \mathbb{E}_{\tau \sim \pi_{\boldsymbol{\theta}}}
\left[
\sum_{t=0}^T \gamma^{t}\, A(s_t,a_t)\, \big(e_a-\pi_{\boldsymbol{\theta}}(s_t)\big)\otimes h(s_t)
\right],
\end{equation}
where \(e_a\in\mathbb{R}^{K}\), $K = \text{dim}(\mathcal{V})$, denotes the one-hot vector of the realized action \(a_t\) at time \(t\) (i.e., \(e_a=e_{a_t}\)), \(\pi_{\boldsymbol{\theta}}(s_t)\in\mathbb{R}^{K}\) is the vector of action probabilities at \(s_t\), and \(\otimes\) denotes the Kronecker product.
\end{proposition}
\begin{proof}
Starting from the advantage version of Equation \ref{eq:obj}, the policy gradient with respect to $\boldsymbol{\psi}$ is given by
\begin{equation}\label{eq:gradpsi}
    \tilde{g}(\boldsymbol{\psi})
= \mathbb{E}_{\tau \sim \pi_{\boldsymbol{\theta}}}
\left[
\sum_{t=0}^T \gamma^{t}\, A(s_t,a_t)\, \nabla_{\boldsymbol{\psi}}\log\pi_{\boldsymbol{\theta}}(a_t\mid s_t)
\right].
\end{equation}
With logits $f(s_{t})=Wh_{\bar{\boldsymbol{\theta}}}(s_{t})$, the Jacobian of the log-softmax with respect to $f(s_{t})$ is:
\begin{equation}
    \frac{\partial \log \pi_{\boldsymbol{\theta}}(a\mid s)}{\partial f(s_{t})}
= e_a - \pi_{\boldsymbol{\theta}}(s_{t})\ \in\ \mathbb{R}^{K}.
\end{equation}
Vectorizing $W$ gives:
\begin{equation}
    \frac{\partial f(s_{t})}{\partial \boldsymbol{\psi}}
= I_K \otimes h_{\bar{\boldsymbol{\theta}}}(s_{t})^{\top}\ \in\ \mathbb{R}^{K\times Kd}.
\end{equation}
By the chain rule,
\[
\nabla_{\boldsymbol{\psi}}\log \pi_{\boldsymbol{\theta}}(a\mid s)
= \Big(e_a-\pi_{\boldsymbol{\theta}}(s_{t})\Big)^{\top}\!\big(I_K\otimes h_{\bar{\boldsymbol{\theta}}}(s_{t})^{\top}\big)
= \big(e_a-\pi_{\boldsymbol{\theta}}(s_{t})\big)\otimes h_{\bar{\boldsymbol{\theta}}}(s_{t}),
\]
where we used standard Kronecker product identities to obtain a vector in \(\mathbb{R}^{Kd}\). Plugging the expression for $\nabla_{\boldsymbol{\psi}}\log \pi_{\boldsymbol{\theta}}(a_t\mid s_t)$ into Equation \ref{eq:gradpsi} yields
\begin{equation}
    \tilde{g}(\boldsymbol{\psi})
= \mathbb{E}_{\tau \sim \pi_{\boldsymbol{\theta}}}
\left[
\sum_{t=0}^T \gamma^{t}\, A(s_t,a_t)\, \big(e_a-\pi_{\boldsymbol{\theta}}(s_t)\big)\otimes h(s_t)
\right].
\end{equation}
\end{proof}
\textbf{The Hessian of the Objective}. For the Hessian, we start by extending the PG Theorem for Hessians:
\begin{lemma}[Hessian of the Policy Gradient]\label{lemma:hessianpg}
Let $\pi_{\boldsymbol{\theta}}(a \mid s)$ be a differentiable stochastic policy and consider the discounted policy gradient objective

\begin{equation}
    J(\boldsymbol{\theta}) \;=\; \mathbb{E}_{\tau \sim \pi_{\boldsymbol{\theta}}} \!\left[ \sum_{t=0}^T \gamma^t A(s_t,a_t) \log \pi_{\boldsymbol{\theta}}(a_t \mid s_t) \right],
\end{equation}
where $A(s_t,a_t)$ is the advantage function at time $t$. Then, the Hessian of $J(\boldsymbol{\theta})$ is given by

\begin{equation}\label{eq:hessianadv}
    \nabla^2_{\boldsymbol{\theta}} J(\boldsymbol{\theta}) \;=\; 
\mathbb{E}_{\tau \sim \pi_{\boldsymbol{\theta}}} \!\left[ \sum_{t=0}^T \gamma^t A(s_t,a_t) 
\left( \nabla_{\boldsymbol{\theta}} \log \pi_{\boldsymbol{\theta}}(a_t \mid s_t) \nabla_{\boldsymbol{\theta}} \log \pi_{\boldsymbol{\theta}}(a_t \mid s_t)^\top
+ \nabla^2_{\boldsymbol{\theta}} \log \pi_{\boldsymbol{\theta}}(a_t \mid s_t) \right) \right].
\end{equation}
\end{lemma}
\begin{proof}
Taking the first derivative of $J(\boldsymbol{\theta})$, we obtain
\begin{equation}
    \nabla_{\boldsymbol{\theta}} J(\boldsymbol{\theta}) \;=\; 
\mathbb{E}_{\tau \sim \pi_{\boldsymbol{\theta}}} \!\left[ \sum_{t=0}^T \gamma^t A(s_t,a_t)\, \nabla_{\boldsymbol{\theta}} \log \pi_{\boldsymbol{\theta}}(a_t \mid s_t) \right].
\end{equation}
Differentiating once more yields
\begin{equation}
    \nabla^2_{\boldsymbol{\theta}} J(\boldsymbol{\theta}) 
= \nabla_{\boldsymbol{\theta}} \, \mathbb{E}_{\tau \sim \pi_{\boldsymbol{\theta}}} \!\left[ \sum_{t=0}^T \gamma^t A(s_t,a_t)\, \nabla_{\boldsymbol{\theta}} \log \pi_{\boldsymbol{\theta}}(a_t \mid s_t) \right].
\end{equation}
Expanding the expectation explicitly over state--action pairs weighted by the discounted state distribution $d^\pi_\gamma(s_{t})$ gives
\begin{equation}
    \nabla^2_{\boldsymbol{\theta}} J(\boldsymbol{\theta})
= \sum_{s} d^\pi_\gamma(s_{t}) \sum_{a} \nabla_{\boldsymbol{\theta}} \Big[ \pi_{\boldsymbol{\theta}}(a \mid s)\, A(s,a)\, \nabla_{\boldsymbol{\theta}} \log \pi_{\boldsymbol{\theta}}(a \mid s) \Big].
\end{equation}
Applying the product rule, we obtain
\begin{equation}
    \nabla^2_{\boldsymbol{\theta}} J(\boldsymbol{\theta})
= \sum_{s} d^\pi_\gamma(s_{t}) \sum_{a} \pi_{\boldsymbol{\theta}}(a \mid s) A(s,a) \Big( \nabla_{\boldsymbol{\theta}} \log \pi_{\boldsymbol{\theta}}(a \mid s) \nabla_{\boldsymbol{\theta}} \log \pi_{\boldsymbol{\theta}}(a \mid s)^\top + \nabla^2_{\boldsymbol{\theta}} \log \pi_{\boldsymbol{\theta}}(a \mid s) \Big).
\end{equation}
Rewriting in expectation form gives the final result:
\begin{equation}
    \nabla^2_{\boldsymbol{\theta}} J(\boldsymbol{\theta}) \;=\; 
\mathbb{E}_{\tau \sim \pi_{\boldsymbol{\theta}}} \!\left[ \sum_{t=0}^T \gamma^t A(s_t,a_t) 
\left( \nabla_{\boldsymbol{\theta}} \log \pi_{\boldsymbol{\theta}}(a_t \mid s_t) \nabla_{\boldsymbol{\theta}} \log \pi_{\boldsymbol{\theta}}(a_t \mid s_t)^\top
+ \nabla^2_{\boldsymbol{\theta}} \log \pi_{\boldsymbol{\theta}}(a_t \mid s_t) \right) \right].
\end{equation}
\end{proof}
Now, we can state the Hessian form under the last-layer model:
\begin{proposition}  
[Hessian under Last-Layer Model] Let us consider a softmax policy $\pi_{\boldsymbol{\theta}}(a \mid s)
= \frac{\exp\!\big(f_{\boldsymbol{\theta}}(s,a)\big)}{\sum_{a'} \exp\!\big(f_{\boldsymbol{\theta}}(s,a')\big)}$. Let us also denote the pre-softmax layer by $f(s_{t})=Wh_{\bar{\boldsymbol{\theta}}}(s_{t}), W\in\mathbb{R}^{K\times d},\; h_{\bar{\boldsymbol{\theta}}}(s_{t})\in\mathbb{R}^{d}$. Define $\boldsymbol{\psi} \coloneqq \mathrm{vec}(W)\in\mathbb{R}^{Kd}$, with $\boldsymbol{\theta} = (\bar{\boldsymbol{\theta}}, \boldsymbol{\psi})$, $K = \text{dim}(\mathcal{V})
$. Then, the Hessian of the discounted policy gradient objective
\begin{equation}
    J(\theta) = \mathbb{E}_{\tau \sim \pi_{\boldsymbol{\theta}}} \!\left[ \sum_{t=0}^T \gamma^t A(s_t,a_t) \log \pi_{\boldsymbol{\theta}}(a_t \mid s_t) \right]
\end{equation}
is given by
\begin{equation}
   \tilde{H}(\boldsymbol{\psi}) = \nabla^2_{\boldsymbol{\psi}} J(\theta) 
= \mathbb{E}_{\tau \sim \pi_{\boldsymbol{\theta}}} \!\left[ \sum_{t=0}^{T}\gamma^{t}
A(s,a) \Big( (e_a - \pi_{\boldsymbol{\theta}} (s_{t}))(e_a - \pi_{\boldsymbol{\theta}} (s_{t}))^\top - F(s_{t}) \Big) \otimes h_{\bar{\boldsymbol{\theta}}}(s_{t}) h_{\bar{\boldsymbol{\theta}}}(s_{t})^\top
\right],
\end{equation}
where $e_a \in \mathbb{R}^K$ is the one-hot vector of action $a$, $\pi_{\boldsymbol{\theta}}(s_{t}) \in \mathbb{R}^K$ is the vector of action probabilities, and $F(s_{t}) := \mathrm{diag}(\pi_{\boldsymbol{\theta}}(s_{t})) - \pi_{\boldsymbol{\theta}}(s_{t})\pi_{\boldsymbol{\theta}}(s_{t})^\top$ is the Fisher information matrix at state $s_{t}$.
\end{proposition}
\begin{proof}
From Proposition \ref{prop:grad}, 
\begin{equation}
    \nabla_{\boldsymbol{\psi}} \log \pi_{\boldsymbol{\theta}}(a _{t} \mid s_{t}) 
= (e_a - \pi_{\boldsymbol{\theta}}(s_{t})) \otimes h_{\bar{\boldsymbol{\theta}}}(s_{t}).
\end{equation}
Hence, the outer product is
\begin{align}
    \nabla_{\boldsymbol{\psi}} \log \pi_{\boldsymbol{\theta}}(a _{t} \mid s_{t})\,\nabla_{\boldsymbol{\psi}} \log \pi_{\boldsymbol{\theta}}(a _{t} \mid s_{t})^\top &=  \\
&= \big((e_a - \pi_{\boldsymbol{\theta}}(s_{t})) \otimes h_{\bar{\boldsymbol{\theta}}}(s_{t})\big)\big((e_a - \pi_{\boldsymbol{\theta}}(s_{t})) \otimes h_{\bar{\boldsymbol{\theta}}}(s_{t})\big)^\top \nonumber \\
&= (e_a - \pi_{\boldsymbol{\theta}}(s_{t}))(e_a - \pi_{\boldsymbol{\theta}}(s_{t}))^\top \otimes h_{\bar{\boldsymbol{\theta}}}(s_{t}) h_{\bar{\boldsymbol{\theta}}}(s_{t})^\top, \nonumber
\end{align}
where we applied the identity $
(u \otimes v)(u \otimes v)^\top = (u u^\top) \otimes (v v^\top)
$. Next, we compute the second derivative. Since $
\nabla_{\boldsymbol{\psi}} \log \pi_{\boldsymbol{\theta}}(a _{t} \mid s_{t}) 
= (e_a - \pi_{\boldsymbol{\theta}}(s_{t})) \otimes h_{\bar{\boldsymbol{\theta}}}(s_{t})$,
it follows that
\begin{equation}
    \nabla^2_{\boldsymbol{\psi}} \log \pi_{\boldsymbol{\theta}}(a _{t} \mid s_{t}) = -\nabla_{\boldsymbol{\psi}} \pi_{\boldsymbol{\theta}}(s_{t}) \otimes h_{\bar{\boldsymbol{\theta}}}(s_{t}).
\end{equation}
Using $\nabla \pi_{\boldsymbol{\theta}}(s_{t}) = \big(\mathrm{diag}(\pi_{\boldsymbol{\theta}}(s_{t})) - \pi_{\boldsymbol{\theta}}(s_{t})\pi_{\boldsymbol{\theta}}(s_{t})^\top\big) \otimes h_{\bar{\boldsymbol{\theta}}}(s_{t})$, we obtain
\begin{align}
    \nabla^2_{\boldsymbol{\psi}} \log \pi_{\boldsymbol{\theta}}(a _{t} \mid s_{t}) &= \\
&= - \big(\mathrm{diag}(\pi_{\boldsymbol{\theta}}(s_{t})) - \pi_{\boldsymbol{\theta}}(s_{t})\pi_{\boldsymbol{\theta}}(s_{t})^\top\big) \otimes h_{\bar{\boldsymbol{\theta}}}(s_{t}) h_{\bar{\boldsymbol{\theta}}}(s_{t})^\top \nonumber \\
&= -F(s_{t}) \otimes h_{\bar{\boldsymbol{\theta}}}(s_{t}) h_{\bar{\boldsymbol{\theta}}}(s_{t})^\top.
\end{align}
Finally, substituting both terms into the general Hessian expression from Lemma \ref{lemma:hessianpg},
\[
\nabla^2_{\boldsymbol{\psi}} J({\boldsymbol{\psi}}) 
= \mathbb{E}_{s,a \sim \pi_{\boldsymbol{\psi}}} \!\left[ 
A(s,a) \left( \nabla_{\boldsymbol{\psi}} \log \pi_{\boldsymbol{\theta}}(a _{t} \mid s_{t})\,\nabla_{\boldsymbol{\psi}} \log \pi_{\boldsymbol{\theta}}(a _{t} \mid s_{t})^\top 
+ \nabla^2_{\boldsymbol{\psi}} \log \pi_{\boldsymbol{\theta}}(a _{t} \mid s_{t}) \right)
\right],
\]
yields:
\begin{equation}
   \tilde{H}(\boldsymbol{\psi}) = \nabla^2_{\boldsymbol{\psi}} J(\theta) 
= \mathbb{E}_{\tau \sim \pi_{\boldsymbol{\theta}}} \!\left[ \sum_{t=0}^{T}\gamma^{t}
A(s,a) \Big( (e_a - \pi_{\boldsymbol{\theta}} (s_{t}))(e_a - \pi_{\boldsymbol{\theta}} (s_{t}))^\top - F(s_{t}) \Big) \otimes h_{\bar{\boldsymbol{\theta}}}(s_{t}) h_{\bar{\boldsymbol{\theta}}}(s_{t})^\top
\right],
\end{equation}
\end{proof}
\begin{proposition}[Fisher Information under the Last-Layer Model]
Let us consider a softmax policy $\pi_{\boldsymbol{\theta}}(a \mid s)
= \frac{\exp\!\big(f_{\boldsymbol{\theta}}(s,a)\big)}{\sum_{a'} \exp\!\big(f_{\boldsymbol{\theta}}(s,a')\big)}$. Let us also denote the pre-softmax layer by $f(s_{t}) = W h_{\bar{\boldsymbol{\theta}}}(s_{t})$, $W \in \mathbb{R}^{K \times d}$,$h_{\bar{\boldsymbol{\theta}}}(s_{t}) \in \mathbb{R}^{d}$. Define $\boldsymbol{\psi} \coloneqq \mathrm{vec}(W) \in \mathbb{R}^{Kd}$, with $\boldsymbol{\theta} = (\bar{\boldsymbol{\theta}}, \boldsymbol{\psi})$, $K = \text{dim}(\mathcal{V})$. 
Then, the Fisher information matrix with respect to $\boldsymbol{\psi}$ is
\begin{equation}
   \tilde{F}(\boldsymbol{\psi})
   = \mathbb{E}_{\tau \sim \pi_{\boldsymbol{\theta}}} \!\left[ 
   \big( (e_{a_t} - \pi_{\boldsymbol{\theta}}(s_{t})) (e_{a_t} - \pi_{\boldsymbol{\theta}}(s_{t}))^\top \big) 
   \otimes h_{\bar{\boldsymbol{\theta}}}(s_{t}) h_{\bar{\boldsymbol{\theta}}}(s_{t})^\top
   \right],
\end{equation}
where $e_{a_t} \in \mathbb{R}^K$ is the one-hot vector of the realized action $a_t$, and $\pi_{\boldsymbol{\theta}}(s_{t}) \in \mathbb{R}^K$ is the vector of action probabilities at state $s_t$.
\end{proposition}
\begin{proof}
From Proposition \ref{prop:grad},
\begin{equation}
    \nabla_{\boldsymbol{\psi}} \log \pi_{\boldsymbol{\theta}}(a_t \mid s_t) 
= (e_{a_t} - \pi_{\boldsymbol{\theta}}(s_t)) \otimes h_{\bar{\boldsymbol{\theta}}}(s_t).
\end{equation}
Therefore,
\begin{equation}
    \nabla_{\boldsymbol{\psi}} \log \pi_{\boldsymbol{\theta}}(a_t \mid s_t)\,
\nabla_{\boldsymbol{\psi}} \log \pi_{\boldsymbol{\theta}}(a_t \mid s_t)^\top
= \big((e_{a_t} - \pi_{\boldsymbol{\theta}}(s_t))(e_{a_t} - \pi_{\boldsymbol{\theta}}(s_t))^\top\big) 
\otimes h_{\bar{\boldsymbol{\theta}}}(s_t) h_{\bar{\boldsymbol{\theta}}}(s_t)^\top.
\end{equation}
where the last step follows from the Kronecker identity $(u \otimes x)(v \otimes x)^\top = (uv^\top) \otimes (xx^\top)$. Substituting this into the definition of the discounted Fisher information matrix yields the result.
\end{proof}
\clearpage

\section{Directional Curvatures Computation}\label{app:compdir}

In this section, we present our mechanisms to compute Hessian and Fisher directional curvatures.

\subsection{Directional Fisher Curvature}\label{sec:dirfisher}

For the last-layer parameters $\boldsymbol{\psi}=\mathrm{vec}(W)$ with $W \in \mathbb{R}^{K\times d_{i}}$, $K = \text{dim}(\mathcal{V})$, denote by 
$U := \mathrm{unvec}(\Delta{\boldsymbol{\psi}}) \in \mathbb{R}^{K \times d_{i}}$ the corresponding matrix form of the direction. We aim to compute the curvature of the Fisher information matrix along a direction $\Delta{\boldsymbol{\psi}}$ in parameter space. 
Recall the Fisher information matrix under the Last-Layer Model (Equation \ref{eq:fisher}):

\begin{equation}
   \tilde{F}(\boldsymbol{\psi})
   = \mathbb{E}_{\tau \sim \pi_{\boldsymbol{\theta}}}\!\left[
   \big(u_t u_t^\top \big) \otimes \big(h_t h_t^\top\big)
   \right],
\end{equation}
where $u_t := e_{a_t} - \pi_{\boldsymbol{\theta}}(s_t) \in \mathbb{R}^K$ is the policy error vector and $h_t := h_{\bar{\boldsymbol{\theta}}}(s_t) \in \mathbb{R}^{d_{i}}$ is the feature vector. Using the Kronecker Vector identity 
$\mathrm{vec}(X)^\top (A \otimes B)\,\mathrm{vec}(X) = \mathrm{Tr}(A X B X^\top)$: 

\begin{align}
   \Delta{\boldsymbol{\psi}}^\top \tilde{F}(\boldsymbol{\psi}) \Delta{\boldsymbol{\psi}}
   &= \mathbb{E}_{\tau}\!\left[
   \mathrm{vec}(U)^\top \big(u_t u_t^\top \otimes h_t h_t^\top\big) \mathrm{vec}(U)
   \right] \\
   &= \mathbb{E}_{\tau}\!\left[
   \mathrm{Tr}\!\big(u_t u_t^\top \, U \, h_t h_t^\top U^\top \big)
   \right].
\end{align}
Let $v_t := U h_t \in \mathbb{R}^K$. Then $   \mathrm{Tr}(u_t u_t^\top v_t v_t^\top) = (u_t^\top v_t)^2$. And we obtain:

\begin{equation}
   \Delta{\boldsymbol{\psi}}^\top \tilde{F}(\boldsymbol{\psi}) \Delta{\boldsymbol{\psi}}
   = \mathbb{E}_{\tau \sim \pi_{\boldsymbol{\theta}}}\!\left[
  (u_t^\top v_t)^2
   \right].
\end{equation}

We can estimate the Equation above with samples. Given a batch of $N$ state–action–time samples $\{(s_i,a_i,t_i)\}_{i=1}^N$, an estimator of the curvature is:

\begin{equation}\label{eq:fishersamples}
   \widehat{\Delta{\boldsymbol{\psi}}^{\top} \tilde{F} \Delta{\boldsymbol{\psi}}}
   = \frac{1}{N} \sum_{i=1}^N  
   \big(u_i^\top (\widehat{U} h_i)\big)^2,
\end{equation}
with $u_i = e_{a_i} - \pi_{\boldsymbol{\theta}}(s_i)$ and $h_i = h_{\bar{\boldsymbol{\theta}}}(s_i)$. In practice, $\Delta{\boldsymbol{\psi}}$ itself is typically estimated from data (e.g., as a stochastic gradient direction), hence not strictly deterministic. 
Therefore, estimating Equation \ref{eq:fishersamples} introduces a mild bias as $u_{t}$ and $h_{t}$ are statistically dependent.

\textbf{Cost Analysis.} The computation requires only vector and matrix–vector operations. Per sample, we compute $U h_i$ at cost $\mathcal{O}(Kd)$ and the dot product $u_i^\top v_i$ at cost $\mathcal{O}(K)$, followed by a scalar square. 
In memory, we only store $U$ ($Kd$ parameters) and the per-sample vectors $u_i$ and $h_i$. 
This is dramatically cheaper than materializing the full Fisher matrix $\tilde{F} \in \mathbb{R}^{Kd \times Kd}$, which would require $(Kd)^2$ entries. 

\subsection{Directional Hessian Curvature}

We now consider the curvature of the Hessian along a direction $\Delta{\boldsymbol{\psi}}$. We also assume the same notation as in subsection \ref{sec:dirfisher}
Recall the Hessian under the Last-Layer model (Equation \ref{eq:hessian}):

\begin{equation}
   \tilde{H}(\boldsymbol{\psi})
   = \mathbb{E}_{\tau \sim \pi_{\boldsymbol{\theta}}}\!\left[ 
   A(s,a) \Big( u_{t}u_{t}^\top - F(s) \Big) 
   \otimes h_{\bar{\boldsymbol{\theta}}}(s) h_{t}h_{t}^\top
   \right],
\end{equation}
where $F(s) = \mathrm{diag}(\pi_{\boldsymbol{\theta}}(s)) - \pi_{\boldsymbol{\theta}}(s)\pi_{\boldsymbol{\theta}}(s)^\top$ is the Fisher matrix at state $s$, $u_t := e_{a_t} - \pi_{\boldsymbol{\theta}}(s_t) \in \mathbb{R}^K$ is the policy error vector and $h_t := h_{\bar{\boldsymbol{\theta}}}(s_t) \in \mathbb{R}^{d_{i}}$ is the feature vector.

The directional curvature along $\Delta{\boldsymbol{\psi}}$ is

\begin{align}
\Delta{\boldsymbol{\psi}}^\top \tilde{H}(\boldsymbol{\psi}) \Delta{\boldsymbol{\psi}}
&= 
\mathbb{E}_{\tau}\!\left[
\sum_{t=0}^{T}\gamma^t\,
A(s_t,a_t)\,
\mathrm{vec}(U)^\top
\Big( (u_tu_t^\top - F(s_t)) \otimes h_t h_t^\top \Big)
\mathrm{vec}(U)
\right] .
\end{align}
Using the Kronecker–Vector identity $\mathrm{vec}(X)^\top (A\otimes B)\,\mathrm{vec}(X)=\mathrm{Tr}(A X B X^\top)$, we obtain:
\begin{align}
\Delta_{\boldsymbol{\psi}}^\top \tilde{H}(\boldsymbol{\psi}) \Delta_{\boldsymbol{\psi}}
&=
\mathbb{E}_{\tau}\!\left[
\sum_{t=0}^{T}\gamma^t\,
A(s_t,a_t)\,
\Big(
\mathrm{Tr}(u_tu_t^\top\, U h_t h_t^\top U^\top)
-
\mathrm{Tr}(F(s_t)\, U h_t h_t^\top U^\top)
\Big)
\right] .
\end{align}
Let $v_t := U h_t$. Then the two traces simplify via
\[
\mathrm{Tr}(u_tu_t^\top v_t v_t^\top) = (u_t^\top v_t)^2,
\qquad
\mathrm{Tr}(F(s_t)\, v_t v_t^\top) = v_t^\top F(s_t)\, v_t,
\]
where the first equality uses
$u u^\top v v^\top = (u^\top v)\, u v^\top$ and $\mathrm{Tr}(ab^\top)=b^\top a$,
and the second uses $\mathrm{Tr}(A\,xx^\top)=x^\top A x$.
Hence,
\begin{equation}
\Delta_{\boldsymbol{\psi}}^\top \tilde{H}(\boldsymbol{\psi}) \Delta_{\boldsymbol{\psi}}
=
\mathbb{E}_{\tau \sim \pi_{\boldsymbol{\theta}}}\!\left[
\sum_{t=0}^{T}\gamma^t\,
A(s_t,a_t)\,
\Big( (u_t^\top v_t)^2 \;-\; v_t^\top F(s_t)\, v_t \Big)
\right].
\end{equation}

We can estimate the Equation above via samples, noting the same remarks as in subsection \ref{sec:dirfisher}.
The sample-based estimator is
\begin{equation}\label{eq:hessiansamples}
\widehat{\Delta{\boldsymbol{\psi}}^\top \tilde{H} \Delta{\boldsymbol{\psi}}}
=
\frac{1}{N}\sum_{i=1}^N 
\gamma^{t_i} A(s_i,a_i)\,
\Big( (u_i^\top \widehat{v_i})^2 - \widehat{v_i}^\top F(s_i)\, \widehat{v_i} \Big),
\quad
u_i=e_{a_i}-\pi_{\boldsymbol{\theta}}(s_i),\;\;
\widehat{v_i}=\widehat{U} h_{\bar{\boldsymbol{\theta}}}(s_i).
\end{equation}

\textbf{Cost Analysis.} The computation again only involves vectors and matrix–vector operations. Per sample, we compute $v_{t} = U h_{t}$ at cost $\mathcal{O}(Kd)$, then $( u_{t}^\top v_{t} )^2$ at cost $O(K)$. 
The second term requires an analogous computation to the Fisher case in subsection \ref{sec:dirfisher}. Hence, the complexity remains $O(Kd)$ per sample, and the memory cost is linear in $K$ and $d$, avoiding materialization of the full Hessian $\tilde{H} \in \mathbb{R}^{Kd \times Kd}$.

\clearpage

\section{Monotonic Policy Improvement under CAPO}\label{app:capo}

In this section, we formalize the conditions of monotonic improvement under CAPO.

\begin{assumption}[Bounded curvature and step norms]
\label{ass:curvature-and-steps}
Let $\pi_{\boldsymbol{\theta}}$ be a differentiable policy with objective $J({\boldsymbol{\theta}})$.
Write $g({\boldsymbol{\theta}}) = \nabla_{\boldsymbol{\theta}} J({\boldsymbol{\theta}})$, $
H(\boldsymbol{\theta}) \;=\; \nabla_{\boldsymbol{\theta}}^2 J({\boldsymbol{\theta}})$, and
$F(\boldsymbol{\theta}) \;=\; \mathbb{E}_{s\sim d_\pi,\,a\sim\pi_{\boldsymbol{\theta}}(\cdot|s)}\!\big[
\nabla_{\boldsymbol{\theta}}\log\pi_{\boldsymbol{\theta}}(a|s)\,\nabla_{\boldsymbol{\theta}}\log\pi_{\boldsymbol{\theta}}(a|s)^\top
\big]$. 
For $\Delta\boldsymbol{\theta}\in\mathbb{R}^d$ define the quadratic diagnostics
\begin{equation}
    m_H(\Delta\boldsymbol{\theta})\;:=\; g(\boldsymbol{\theta})^\top \Delta\boldsymbol{\theta} + \tfrac12\,\Delta\boldsymbol{\theta}^\top H(\boldsymbol{\theta}) \Delta\boldsymbol{\theta},
\qquad
m_F(\Delta\boldsymbol{\theta})\;:=\; \tfrac12\,\Delta\boldsymbol{\theta}^\top F(\boldsymbol{\theta}) \Delta\boldsymbol{\theta}.
\end{equation}

Assume:
\begin{enumerate}
  \item[(i)] (\textbf{Hessian operator norm bound}) $\|H(\boldsymbol{\theta})\|_{\op} \le M$ for some finite $M>0$, where $\|H(\boldsymbol{\theta})\|_{\op} := \sup_{x\neq 0}\frac{\|H(\boldsymbol{\theta})x\|}{\|x\|}$.
  \item[(ii)] (\textbf{Per-candidate step bound}) Every candidate update considered by the algorithm satisfies $\|\Delta\boldsymbol{\theta}\|\le r$ for some $r>0$.
\end{enumerate}
\end{assumption}

\textbf{Remarks}. The step norm bound is standard in practice, since learning rates, clipping, or
trust-region constraints ensure $\|\Delta\boldsymbol{\theta}\|\le r$. The Hessian bound $\|H(\boldsymbol{\theta})\|_{\op}\le M$ is more restrictive globally, but over any
compact region of parameter space visited by the algorithm, continuity of $H(\boldsymbol{\theta})$
implies a finite $M$.

\begin{lemma}[Surrogate--true performance gap]\label{lem:trpo-gap}
For any policies $\pi$ and $\pi'$, with $D_{\mathrm{KL}}(\pi\Vert\pi')$ the average forward KL under $d_\pi$,
\begin{equation}\label{eq:cpo}
    J(\pi') \ \ge\ L_\pi(\pi') \ -\ C\,\sqrt{D_{\mathrm{KL}}(\pi\Vert\pi')},\qquad
    C=\frac{2\gamma}{(1-\gamma)^2}\,\epsilon\,\sqrt{2},
\end{equation}
where $|A^\pi(s,a)|\le \epsilon$ with $\epsilon$ finite, and $L_\pi(\pi') := J(\pi) \;+\; 
\mathbb{E}_{s\sim d_\pi,\;a\sim \pi'(\cdot|s)}[A_\pi(s,a)].$
Moreover, writing $\pi=\pi_{\boldsymbol{\theta}}$ and $\pi'=\pi_{\boldsymbol{\theta}+\Delta\boldsymbol{\theta}}$ for a parameter step $\Delta\boldsymbol{\theta}$,
\begin{equation}\label{eq:expandobj}
    L_{\pi_{\boldsymbol{\theta}}}(\pi_{\boldsymbol{\theta}'})-J(\pi_{\boldsymbol{\theta}})
    \;=\; g(\boldsymbol{\theta})^\top \Delta\boldsymbol{\theta} \;+\; \tfrac12\,\Delta\boldsymbol{\theta}^\top H(\boldsymbol{\theta})\,\Delta\boldsymbol{\theta} \;+\; o(\|\Delta\boldsymbol{\theta}\|^2).
\end{equation}
\end{lemma}

\begin{proof}
The proof of \eqref{eq:cpo} is in \cite{pmlr-v70-achiam17a}. For Equation \ref{eq:expandobj}, we define $\Psi(\boldsymbol{\theta}') := L_{\pi_{\boldsymbol{\theta}}}(\pi_{\boldsymbol{\theta}'})$. Note that $\Psi(\boldsymbol{\theta}) \;=\; J(\pi_{\boldsymbol{\theta}})$. Now compute the gradient of $\Psi$ at $\boldsymbol{\theta}'=\boldsymbol{\theta}$:
\begin{align}
\nabla_{\boldsymbol{\theta}'} \Psi(\boldsymbol{\theta}')\big|_{\boldsymbol{\theta}'=\boldsymbol{\theta}}
&= \nabla_{\boldsymbol{\theta}'} 
  \mathbb{E}_{s\sim d_\pi,\;a\sim \pi_{\boldsymbol{\theta}'}(\cdot|s)}[A_\pi(s,a)]
  \Big|_{\boldsymbol{\theta}'=\boldsymbol{\theta}} \nonumber \\
&= \mathbb{E}_{s\sim d_\pi,\;a\sim \pi_{\boldsymbol{\theta}}}
   \!\left[A_\pi(s,a)\,\nabla_{\boldsymbol{\theta}'}\log \pi_{\boldsymbol{\theta}'}(a|s)\right]_{\boldsymbol{\theta}'=\boldsymbol{\theta}} \nonumber \\
&= \mathbb{E}_{s\sim d_\pi,\;a\sim \pi}
   \!\left[A_\pi(s,a)\,\nabla_{\boldsymbol{\theta}}\log \pi_{\boldsymbol{\theta}}(a|s)\right] =: g(\boldsymbol{\theta}),
\end{align}
where $g(\boldsymbol{\theta})$ is precisely the policy gradient. Differentiate once more:
\begin{align*}
\nabla_{\boldsymbol{\theta}'}^2 \Psi(\boldsymbol{\theta}')\big|_{\boldsymbol{\theta}'=\boldsymbol{\theta}}
&= \mathbb{E}_{s\sim d_{\pi_{\boldsymbol{\theta}}},\,a\sim \pi_{\boldsymbol{\theta}'}(\cdot\mid s)}
   \!\Big[A_{\pi_{\boldsymbol{\theta}}}(s,a)\,\nabla_{\boldsymbol{\theta}'}^2\log \pi_{\boldsymbol{\theta}'}(a\mid s)\Big]_{\boldsymbol{\theta}'={\boldsymbol{\theta}}} \\
&\quad + \mathbb{E}_{s\sim d_{\pi_{\boldsymbol{\theta}}},\,a\sim \pi_{\boldsymbol{\theta}'}(\cdot\mid s)}
   \!\Big[A_{\pi_{\boldsymbol{\theta}}}(s,a)\,\nabla_{\boldsymbol{\theta}'}\log \pi_{\boldsymbol{\theta}'}(a\mid s)\,
            \nabla_{\boldsymbol{\theta}'}\log \pi_{\boldsymbol{\theta}'}(a\mid s)^\top\Big]_{\boldsymbol{\theta}'={\boldsymbol{\theta}}}. \\
 &:= H(\boldsymbol{{\boldsymbol{\theta}}}).
\end{align*}

By the second-order Taylor expansion,
\begin{equation}\label{eq:sec_exp}
\Psi({\boldsymbol{\theta}}+\Delta{\boldsymbol{\theta}}) 
= \Psi({\boldsymbol{\theta}}) \;+\; g(\boldsymbol{\theta})^\top \Delta{\boldsymbol{\theta}} \;+\; \tfrac12\,\Delta{\boldsymbol{\theta}}^\top H(\boldsymbol{\theta})\,\Delta{\boldsymbol{\theta}} \;+\; o(\|\Delta{\boldsymbol{\theta}}\|^2),
\end{equation}

which is exactly \eqref{eq:expandobj}. 
\end{proof}

\begin{theorem}[Monotonic improvement under CAPO, restated]
\label{thm:capo-certified}
Fix thresholds $\delta_H>0$ and $\delta_F>0$. Let $\mathcal{B}$ be a batch of sampled trajectories. Split $\mathcal{B}$ into disjoint $N$ subsets  $b_{i} \subset \mathcal{B}$, and propose candidate subset updates $\{\Delta\theta_i\}_{i:N}$. Retain those satisfying:
\begin{equation}
    m_H(\Delta\theta_i)\;\ge\; \delta_{H} = \omega+\tfrac12Mr^2,\qquad 
m_F(\Delta\theta_i)\;\le\;\delta_F,
\end{equation}
with $\omega >  0$ and $M$, $r$ defined as in Assumption \ref{ass:curvature-and-steps}. Let $\mathcal{B}_{acc}$ denote the superset of the B accepted subsets, and define the aggregated update:$    \Delta\theta \;=\; \frac{1}{B}\sum_{i\in\mathcal{B}_{acc}}\Delta\theta_i$.
Then, for two policies $\pi_{\boldsymbol{\theta}}$ and $\pi_{\boldsymbol{\theta} + \Delta \boldsymbol{\theta}}$, we obtain:
\begin{equation}
    J(\pi_{\theta+\Delta\theta}) - J(\pi_\theta) 
\;\ge\; \omega \;-\; C\,\sqrt{\delta_F}.
\end{equation}
Thus choosing $\omega \ge C\sqrt{\delta_F}$ guarantees monotonic improvement: $J(\pi_{\theta+\Delta\theta}) \ge J(\pi_\theta)$.
\end{theorem}

\begin{proof} We first establish bounds in the global Fisher and Hessian directional curvatures.

\textbf{Fisher global bound.}
Since $F\succeq 0$, the quadratic form $\phi(u):=u^\top F u$ is convex. 
Thus:
\begin{equation}
\label{eq:fisher-convex-identity}
\Delta\theta^\top F \Delta\theta =
\Big(\frac{1}{B}\sum_{i\in\mathcal{B}_{acc}}\Delta\theta_i \Big)^\top 
F \Big(\frac{1}{B}\sum_{i\in\mathcal{B}_{acc}}\Delta\theta_i\Big)
\;\le\; \frac{1}{B}\sum_{i\in\mathcal{B}_{acc}}\Delta\theta_i^\top F \Delta\theta_i.
\end{equation}
The inequality above follows from: 
\begin{align}
\frac{1}{B}\sum_{i\in\mathcal{B}_{acc}}\Delta\theta_i^\top F \Delta\theta_i 
\;-\;
\Big(\frac{1}{B}\sum_{i\in\mathcal{B}_{acc}}\Delta\theta_i\Big)^\top 
F \Big(\frac{1}{B}\sum_{i\in\mathcal{B}_{acc}}\Delta\theta_i\Big) \\
= \frac{1}{2B^2}\sum_{i,j\in\mathcal{B}_{acc}}
(\Delta\theta_i-\Delta\theta_j)^\top F (\Delta\theta_i-\Delta\theta_j) \ge 0,
\end{align}
because $F\succeq 0$ implies each summand is nonnegative. Hence:
\begin{equation}
    \Delta\theta^\top F \Delta\theta
\;\le\; \frac{1}{B}\sum_{i\in\mathcal{B}_{acc}}\Delta\theta_i^\top F \Delta\theta_i
\;\le\; \frac{1}{B}\sum_{i\in\mathcal{B}_{acc}} 2\, m_F(\Delta\theta_i)
\;\le\; 2\,\delta_F.
\end{equation}
\textbf{Hessian global bound.} Expanding $m_H(\Delta\boldsymbol{\theta})$:
\begin{equation}
\label{eq:mH-expand}
\begin{aligned}
m_H(\Delta\boldsymbol{\theta})
&= g(\boldsymbol{\theta})^\top \Delta\boldsymbol{\theta} + \tfrac12\,\Delta\boldsymbol{\theta}^\top H \Delta\boldsymbol{\theta} \\
&= g(\boldsymbol{\theta})^\top\!\Big(\tfrac1B \sum_{i\in\mathcal{B}_{acc}} \Delta\boldsymbol{\theta}_i\Big)
   + \frac12\Big(\tfrac1B \sum_{i\in\mathcal{B}_{acc}} \Delta\boldsymbol{\theta}_i\Big)^\top
          H \Big(\tfrac1B \sum_{j\in\mathcal{B}_{acc}} \Delta\boldsymbol{\theta}_j\Big) \\
&= \frac1B \sum_{i\in\mathcal{B}_{acc}} g(\boldsymbol{\theta})^\top \Delta\boldsymbol{\theta}_i
   + \frac{1}{2B^2}\sum_{i,j\in\mathcal{B}_{acc}} \Delta\boldsymbol{\theta}_i^\top H \Delta\boldsymbol{\theta}_j .
\end{aligned}
\end{equation}
We can decompose the quadratic form:
\begin{equation}
\label{eq:split}
\sum_{i,j\in\mathcal{B}_{acc}} \Delta\boldsymbol{\theta}_i^\top H \Delta\boldsymbol{\theta}_j
= \sum_{i\in\mathcal{B}_{acc}} \Delta\boldsymbol{\theta}_i^\top H \Delta\boldsymbol{\theta}_i
+ \sum_{\substack{i,j\in\mathcal{B}_{acc}\\i\neq j}} \Delta\boldsymbol{\theta}_i^\top H \Delta\boldsymbol{\theta}_j .
\end{equation}
Substituting \eqref{eq:split} into \eqref{eq:mH-expand} and grouping yields
\begin{equation}
\label{eq:mH-decomp}
m_H(\Delta\boldsymbol{\theta})
= \frac1B \sum_{i\in\mathcal{B}_{acc}} m_H(\Delta\boldsymbol{\theta}_i)
  - \frac{B-1}{2B^2}\sum_{i\in\mathcal{B}_{acc}} \Delta\boldsymbol{\theta}_i^\top H \Delta\boldsymbol{\theta}_i
  + \frac{1}{2B^2}\sum_{\substack{i,j\in\mathcal{B}_{acc}\\i\neq j}} \Delta\boldsymbol{\theta}_i^\top H \Delta\boldsymbol{\theta}_j .
\end{equation}
By the operator norm bound $\|H\|_{\op}\le M$ and Cauchy--Schwarz,
\[
|\Delta\boldsymbol{\theta}_i^\top H \Delta\boldsymbol{\theta}_j|
\;\le\; M\,\|\Delta\boldsymbol{\theta}_i\|\,\|\Delta\boldsymbol{\theta}_j\| .
\]
Hence, using $\|\Delta\boldsymbol{\theta}_i\|\le r$ for all $i$,
\begin{equation}
\sum_{i\in\mathcal{B}_{acc}} \Delta\boldsymbol{\theta}_i^\top H \Delta\boldsymbol{\theta}_i \;\le\; MB r^2,
\qquad
\sum_{\substack{i,j\in\mathcal{B}_{acc}\\i\neq j}} \Delta\boldsymbol{\theta}_i^\top H \Delta\boldsymbol{\theta}_j
\;\ge\; -M B(B-1) r^2 .
\end{equation}
Substituting into \eqref{eq:mH-decomp},
\begin{equation}
\label{eq:mH-bound}
m_H(\Delta\boldsymbol{\theta})
\;\ge\; \frac1B \sum_{i\in\mathcal{B}_{acc}} m_H(\Delta\boldsymbol{\theta}_i)
- M r^2\Big(1-\tfrac1B\Big).
\end{equation}
If each accepted subset satisfies
$m_H(\Delta\boldsymbol{\theta}_i)\ge\omega + M r^2$, 
then averaging gives $\tfrac1B \sum_{i\in\mathcal{B}_{acc}} m_H(\Delta\boldsymbol{\theta}_i)\ge \omega + M r^2$.
Plugging into \eqref{eq:mH-bound} yields
\begin{equation}
    m_H(\Delta\boldsymbol{\theta}) \;\ge\; \omega + M r^2 - M r^2\Big(1-\tfrac1B\Big)
= \omega + \frac{M r^2}{B}
\;\ge\; \omega .
\end{equation}

From Equations \ref{eq:cpo} and \ref{eq:expandobj} of Lemma \ref{lem:trpo-gap}, we have that:
\begin{equation}
     J(\pi_{\boldsymbol{\theta} + \Delta\boldsymbol{\theta} })-J(\pi_{\boldsymbol{\theta}})
    \;\ge\; \underbrace{g(\boldsymbol{\theta})^\top \Delta\boldsymbol{\theta} \;+\; \tfrac12\,\Delta\boldsymbol{\theta}^\top H(\boldsymbol{\theta})\,\Delta\boldsymbol{\theta}}_{m_{H}(\Delta\boldsymbol{\theta})} \;+\; o(\|\Delta\boldsymbol{\theta}\|^2) - C\underbrace{\sqrt{D_{\mathrm{KL}}(\pi_{\boldsymbol{\theta}}\Vert\pi_{\boldsymbol{\theta} + \Delta\boldsymbol{\theta}}})}_{m_{F}(\Delta\boldsymbol{\theta}) + o(\|\Delta\boldsymbol{\theta}\|^2)}
\end{equation}
Then, using $m_{F}(\Delta\boldsymbol{\theta}) < \delta_{F}$, $m_{H}(\Delta\boldsymbol{\theta}) > \omega$, and assuming the cubic terms negligible,
\begin{equation}
    J(\pi_{\boldsymbol{\theta} + \Delta\boldsymbol{\theta} })-J(\pi_{\boldsymbol{\theta}})
    \;\ge\; \omega - C\sqrt{\delta_{F}}.
\end{equation}
Thus choosing $\omega \ge C\sqrt{\delta_F}$ guarantees monotonic improvement: $J(\pi_{\boldsymbol{\theta} + \Delta\boldsymbol{\theta} }) \ge J(\pi_{\boldsymbol{\theta}})$.
\end{proof}

\clearpage

\section{Pseudocode of CAPO}

In this Appendix, we present CAPO's algorithm.

\begin{algorithm}[H]
\caption{Curvature-Aware Policy Optimization (CAPO)}
\label{alg:capo}
\DontPrintSemicolon
\SetKwInOut{Input}{Input}\SetKwInOut{Output}{Output}
\Input{%
  Policy $\pi_\theta$; batch $\mathcal{B}$ of sampled trajectories;\\
  thresholds $(\delta_F,\ \delta_H,\ \delta_H^{high})$;\\
  optimizer for the last-layer model (e.g., SGD or Adam).\\
}
\Output{Updated policy parameters $\theta$}
\BlankLine
\While{not done}{
  \tcp{Collect data with the current policy}
  Sample a batch $\mathcal{B} = \{\tau\}_{i}^{N}$ of trajectories, $\tau \sim \pi_\theta$.\;
\textbf{Partition} $\mathcal{B}$ into disjoint subsets $\{b_i\}_{i=1}^N$.\;
\For{$i=1,\dots,N$ \textbf{in parallel}}{
  \tcp{Build last-layer meta-model stats on subset $b_i$}
  Estimate model-based gradient $\tilde g(\psi)$ using Equation \ref{eq:gradient};\\
  Propose $\Delta\psi_i$ with the optimizer model
  (e.g., $\Delta\psi_i=\alpha\,\tilde g(\psi)$ for SGD, or Adam’s rule) \\
  Compute directional curvatures $\tfrac12\, \Delta\psi^\top \tilde H(\psi)\, \Delta\psi$, $\Delta\psi^\top \tilde F(\psi)\,\Delta\psi$ as in Appendix \ref{app:compdir}; \\
  Compute objective and policy shifts under the last-layer model:\\
  \hspace{2ex} $m_H(\Delta\psi)\gets \tilde g(\psi)^\top \Delta\psi
     + \tfrac12\, \Delta\psi^\top \tilde H(\psi)\, \Delta\psi$,
  $m_F(\Delta\psi)\gets \tfrac12\,\Delta\psi^\top \tilde F(\psi)\,\Delta\psi$.\;
  \tcp{Local trust-region acceptance test}
  \uIf{$\delta_H \le m_H(\Delta\psi_i)\le \delta_H^{high}$ \textbf{and} $m_F(\Delta\psi_i)\le \delta_F$}{
    Mark subset $b_i$ as \textsc{Accept}; add to $\mathcal{B}_{acc}$.\;
  }\Else{
    \textsc{Reject} $b_i$.\;
  }
}
\tcp{Compute the actual policy update on accepted data}
  \If{$\mathcal{B}_{acc} \neq \varnothing$}{
    Estimate the objective on accepted samples (e.g., GRPO/PPO surrogate):\\
    \hspace{2ex}$\widehat{J}(\theta)\;=\;\texttt{pg-objective}(\pi_\theta;\ \bigcup_{b_i\in \mathcal{B}_{acc}} b_i)$.\;
    \tcp{Policy Gradient and parameter update}
    $\theta \leftarrow \theta + \alpha \,\widehat{\nabla}_\theta J$\;  
  }
  }
\Return{$\theta$}
\end{algorithm}

\clearpage

\section{Reproducibility Statement}\label{app:implementation}

 \textbf{Code Release}. To ensure the reproducibility of our research findings, we release our code at \codeurl. Our implementation is based on PyTorch \citep{paszke2017automatic} and
 HuggingFace \citep{wolf2020huggingfacestransformersstateoftheartnatural}. All baselines are available in the released code. We also plan to publish all the experiments logs in WandB \citep{wandb}.
 
 \textbf{Reproducibility.} We detail our methodology in Sections \ref{sec:compmodel} and \ref{sec:capo} and our experimental setup
 in Section \ref{sec:results}. We provide all hyperparameters used in this work in
 Appendix \ref{sec:hypers}. For all experiments in this paper, we report the results over five seeds with standard errors. For the MATH benchmark, we report in-training performance every step, while for the TEST benchmark set we evaluate checkpoints every 10 learning steps. For better visualization, we applied smoothing with exponential moving average on the curves. All datasets are open-source and available online for academic use.
 
\textbf{Compute Resources}. We execute all RL experiments using 4 NVIDIA H100 GPUs. Each seed in the regime with aggressive updates takes approximately 4 hours, while the standard regime takes approximately one day. Evaluation is done separately in the same hardware, taking approximately 90 minutes per seed.

\textbf{LLM Usage Details.} We use LLMs for paper writing to improve grammar, enhance clarity and writing flow, and assist with code and mathematical iterations. All outputs generated by the LLMs were thoroughly reviewed and verified by the authors to ensure factual accuracy and correctness.

\clearpage
\setlength{\textfloatsep}{10pt}
\setlength{\intextsep}{10pt}%
\section{Computational Cost Analysis}\label{app:compcost}

\textbf{Execution Time}. Table \ref{tab:compcost} reports a breakdown of CAPO’s execution time, including both the model estimations and the masking process. The table shows the average time (in seconds) of each operation, averaged over all learning iterations, measured on our NVIDIA 4×H100 hardware. The total learning iteration time include LLM generations and forward and backward passes. We find that CAPO contributes less than 3\% of the total step time in a learning iteration, resulting in minimal training overhead. Most of the cost arises from computing the Adam gradient and updating its moments, since this also requires computing batch gradients on sparse representations. Lastly, the cost of computing the mask is minimal, below 0.01 seconds.

\textbf{Memory cost.} CAPO uses only volatile GPU memory, since all operations are transient and tensors are discarded after the masking generation. The main memory usage comes from maintaining token-level gradient tensors, which have shape $(N, T, K, D)$, corresponding to batch size, completion length, top-$K$ probabilities, and the number of parameters in the last-layer model. In our experiments, with $N = 24$, $T = 1024$, $K = 50$, and $D = 896$, this amounts to a volatile memory footprint of approximately 2 GB, which is minimal given the scale of LLM training. For comparison, this is significantly less expensive than performing KL regularization, which requires storing an additional copy of the LLM in memory for the reference policy.

\begin{table}[h]
\centering
\begin{tabular}{l c c}
\textbf{Step} & \textbf{Avg. Time (s)} & \textbf{\% of Total} \\
\hline
Learning Iteration (Total) & $135.84$ & $100.00\%$ \\
LLM Generations & $55.50$ & $40.85\%$ \\
\hdashline
Total CAPO time & $\textbf{3.99}$ & $\textbf{2.94\%}$ \\
\quad Compute token-level gradients & $0.04$ & $0.03\%$ \\
\quad Compute Adam token gradients & $0.51$ & $0.38\%$ \\
\quad Compute \& log $m_{H}$ & $0.09$ & $0.07\%$ \\
\quad Compute \& log $m_{F}$ & $0.01$ & $0.01\%$ \\
\quad Update Adam Moments & $3.34$ & $2.46\%$ \\
\quad Compute Hessian Mask & $0.00$ & $0.00\%$ \\
\quad Compute Fisher Mask & $0.00$ & $0.00\%$ \\
\hline
\end{tabular}
\caption{\textbf{Breakdown of the execution time of CAPO}. CAPO contributes less than 3\% of the total step time, resulting in minimal overhead relative to standard training.}
\label{tab:compcost}
\end{table}

\clearpage

\section{Additional Experiments}\label{app:ablations}

\textbf{Ablation of the Optimizer Model.} We conducted an ablation study on the impact of the optimizer representation. This choice reflects a trade-off between step accuracy and computational cost: SGD is cheaper, but the LLM policy is optimized with Adam. Figure \ref{fig:ablation_optimizer} shows the results on the MATH dataset. For CAPO, representing the optimizer with either SGD or Adam yields similar performance. However, for Dr.CAPO and ReinCAPO, the SGD variant is insufficient to prevent policy collapse. This suggests that matching the optimizer representation provides a more robust choice across different setups.

\begin{figure}[h]
\begin{center}
\includegraphics[width=1.0\textwidth]{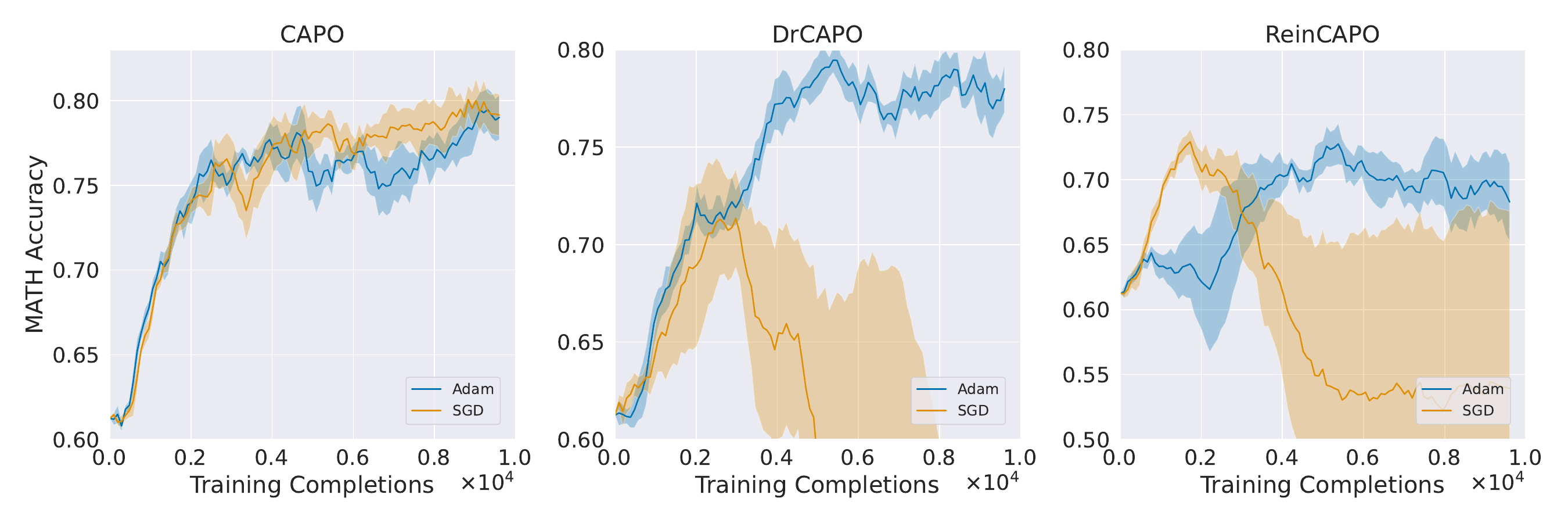}
\caption{\textbf{Ablation study of the optimizer model.}  For CAPO, both representations yield similar performance, whereas for Dr.CAPO and ReinCAPO, only the Adam-based representation prevents policy collapse, indicating that matching the optimizer provides a more robust choice across setups.}
\label{fig:ablation_optimizer}
\end{center}
\end{figure}

\textbf{Is PPO clipping enough to ensure stability?}  PPO clipping \citep{schulman2017proximalpolicyoptimizationalgorithms} is a heuristic designed to prevent large updates by clipping the probability ratio between the current policy and the old policy that collected the data. This raises the question of whether clipping alone is sufficient to avoid policy collapse in our LLM setup. We note that clipping is primarily intended to facilitate off-policy updates, whereas our experiments with on-policy data already reveal instability in current RL methods. Nevertheless, we conducted additional experiments using off-policy data reused for $t$ iterations under different clipping ratios. Figure \ref{fig:clip_study} shows results for two setups: $t=2$ (minimal off-policy shift) and $t=5$ (moderate shift). We find that the standard clipping ratio ($\epsilon=0.2$) does not prevent collapse. More aggressive ratios alleviate instabilities but reduce performance, likely due to the strong bias introduced in the gradients. This trade-off becomes more pronounced as $t$ increases.

\begin{figure}[h]
\begin{center}
\includegraphics[width=0.83\textwidth]{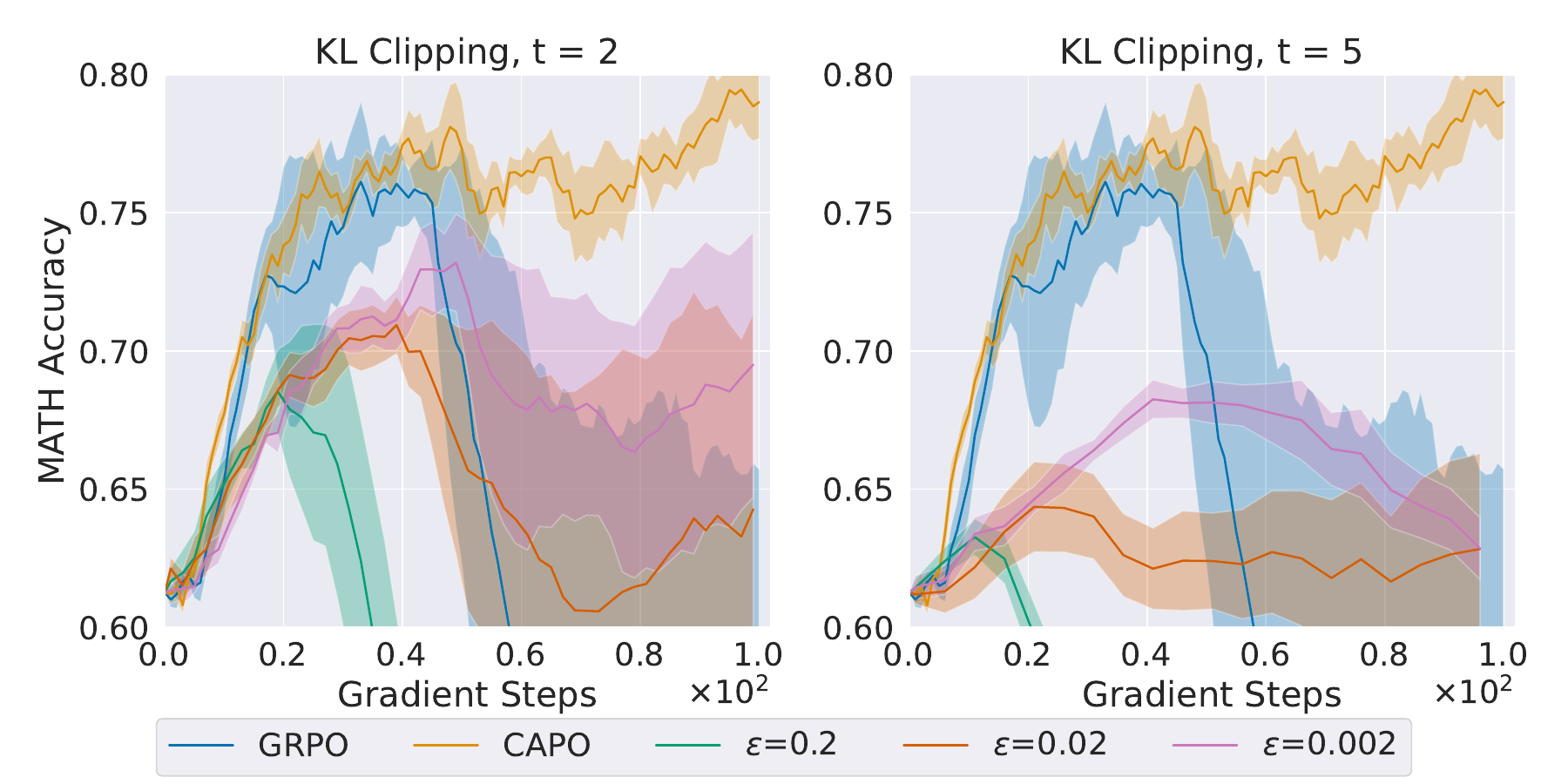}
\caption{\textbf{Effect of ``PPO clipping" on GRPO stability.} Standard clipping ($\epsilon=0.2$) fails to prevent collapse, while more aggressive ratios improve stability but reduce overall performance, with the trade-off becoming more severe as $t$ increases.}
\label{fig:clip_study}
\end{center}
\end{figure}

\textbf{Is KL regularization enough to ensure stability?} Another common strategy to mitigate instabilities is to add a KL regularizer that penalizes deviations from the base policy (see Equation \ref{eq:grpo}). The rationale is that keeping the policy close to the base model may prevent large distributional shifts, such as those associated with policy collapse. In Figure \ref{fig:klreg} (left), we test different levels of regularization. We observe a trend similar to clipping: only stronger regularization ($\beta = 1.0$) helps prevent catastrophic updates, but at the cost of performance.  

A more fundamental limitation of KL regularization becomes evident when examining its gradient:
\begin{equation}
\nabla_\theta \mathcal{D}_{\text{KL}}\!\left(\pi_\theta \,\|\, \pi_{\text{base}}\right)
= \mathbb{E}_{s \sim d^{\pi},\, a \sim \pi_\theta}\!\left[
\nabla_\theta \log \pi_\theta(a \mid s)\,
\Bigg(\log \frac{\pi_\theta(a \mid s)}{\pi_{\text{base}}(a \mid s)} + 1\Bigg)
\right].
\end{equation}

Differentiating through the KL term introduces a multiplicative $\log$ factor, which can produce unbounded gradients. More concretely, as $\pi_{\text{base}}(a \mid s) \to 0$, the gradient magnitude diverges, effectively ``exploding'' the LLM policy gradient. We observe this empirically in Figure \ref{fig:klreg} (right), which shows the maximum gradient norms (before gradient clipping) over training, averaged across seeds. While gradient clipping can reduce the gradient's magnitude, it does not alter its direction, which may still drive the optimization into unstable regions.  

Finally, there are also practical drawbacks to KL regularization. First, it requires storing a full copy of the base model in memory, which has led prior work to abandon the technique \citep{liu2025understandingr1zeroliketrainingcritical}. Second, differentiating KL estimates as loss functions typically yields biased approximations of the true KL gradient \citep{tang2025pitfallskldivergencegradient}.

\begin{figure}[t]
\begin{center}
\includegraphics[width=1.0\textwidth]{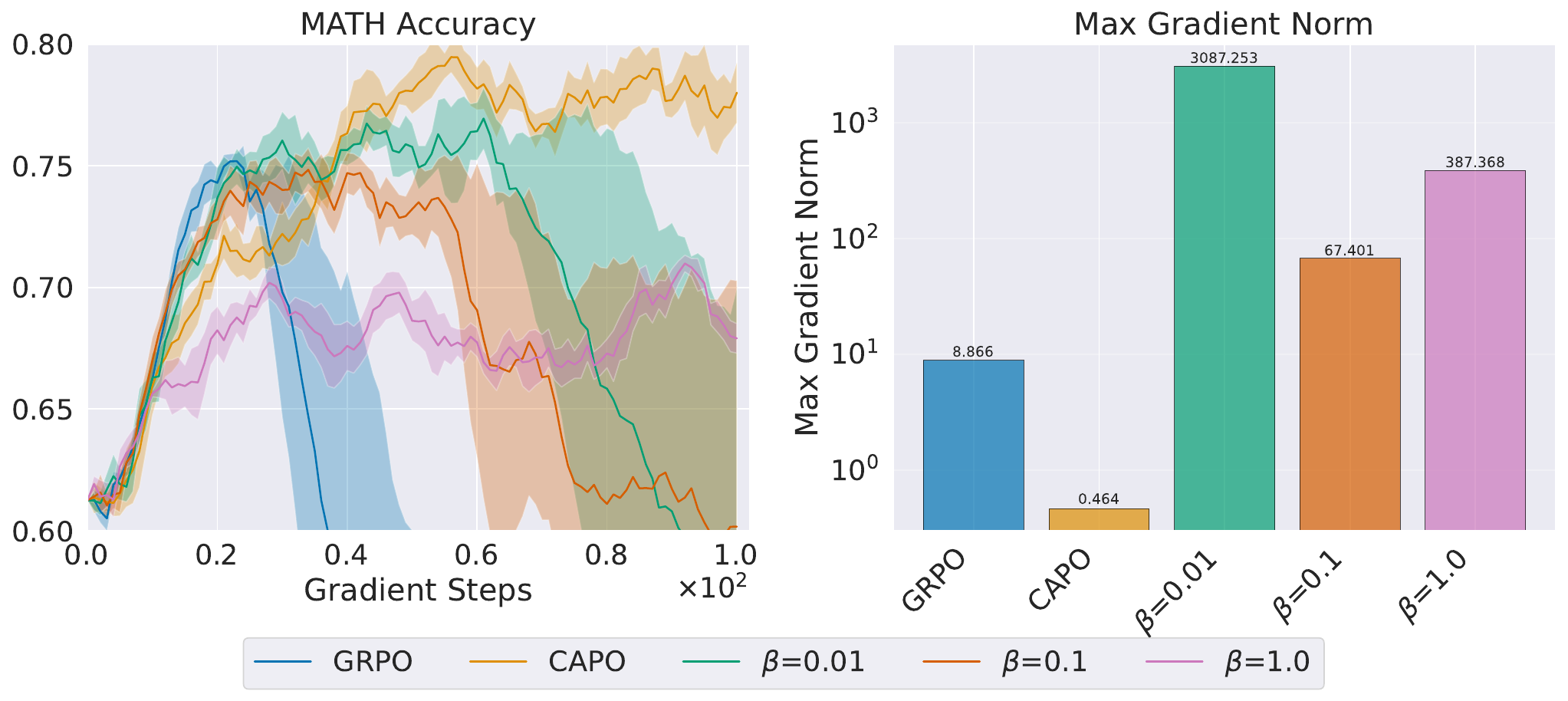}
\caption{\textbf{Effect of KL regularization on GRPO stability.} (Left) Accuracy on the MATH dataset under different levels of KL regularization. Stronger regularization ($\beta=1.0$) reduces instability but degrades performance. 
(Right) Maximum gradient norms (before clipping), averaged across seeds. KL regularization produces unbounded gradients that may drive the optimization into unstable regions.}

\label{fig:klreg}
\end{center}
\end{figure}

\clearpage

\section{Hyperparameters}\label{sec:hypers}

In this section, we present the hyperparameters used in our experiments. Table \ref{tab:hypers_general} lists the hyperparameters common to all training configurations and algorithms. Table \ref{tab:hypers_updates} specifies the learning rate and batch size for the conservative and aggressive setups. Finally, Table \ref{tab:capo_hyper} reports the hyperparameters specific to curvature-aware masking, along with their values for each method. Due to compute budget constraints, we performed manual hyperparameter tuning, primarily searching across different orders of magnitude of both $\delta_{H}$ and $\delta_{F}$. For simplicity, we implemented a single symmetric threshold for the Hessian, i.e., rejecting samples outside the interval $-\delta_H < m_H < \delta_H$.

\begin{table}[h]
\centering
\small
\begin{tabular}{l|l}
 \textbf{Hyperparameter} & \textbf{Value} \\
\hline
\multicolumn{2}{l}{\emph{LLM Generation}} \\
Max Prompt Length & 512 \\
Max Completion Length & 1024 \\
Num Generations per Prompt & 8 \\
Temperature & 0.9 \\
\hline
\multicolumn{2}{l}{\emph{Training}} \\
Gradient Steps & 100 \\
Warmup Ratio & 0.1 \\
Iterations per Batch & 1 \\
Optimizer & Adam \\
LR Scheduler & Cosine \\
KL $\beta$ & 0.0 \\
\end{tabular}
\caption{\textbf{Training Hyperparameters.}}
\label{tab:hypers_general}
\end{table}

\begin{table}[h]
\centering
\small
\begin{tabular}{l|c|c}
 \textbf{Hyperparameter} & \textbf{Standard Setup} & \textbf{Aggressive Setup} \\
\hline
Learning Rate & $3\times10^{-6}$ & $1.5\times10^{-5}$ \\
Total Batch Size & 1152 & 96
\end{tabular}
\caption{\textbf{Hyperparameters for the standard (conservative) and aggressive regimes.}}
\label{tab:hypers_updates}
\end{table}

\begin{table}[h]
\centering
\small
\begin{tabular}{l|c|c|c}
 \textbf{Hyperparameter} & \textbf{CAPO} & \textbf{Dr.CAPO} & \textbf{ReinCAPO} \\
\hline
Hessian $\delta_{H}$ & $10^{-2}$ & $5\times10^{-4}$ &  $10^{-1}$\\
Fisher $\delta_{F}$ & $10^{-4}$ & $10^{-3}$ & $10^{-5}$ \\
\end{tabular}
\caption{\textbf{Curvature-aware masking thresholds for CAPO, Dr.CAPO and ReinCAPO.}}
\label{tab:capo_hyper}
\end{table}

\clearpage

\section{Monotonic Policy Improvement under CAPO in the undiscounted, finite-horizon setting}

Appendix \ref{app:capo} formalizes the conditions under which CAPO guarantees monotonic improvement in the standard discounted, infinite-horizon setting. Although this formulation is general and aligned with prior RL literature, this section extends the analysis to the undiscounted, finite-horizon setting, which better reflects the LLM reasoning setup and is more consistent with the assumptions underlying practical algorithms such as GRPO.

For this analysis, we consider a finite-horizon Markov decision process (MDP) with horizon $T \in \mathbb{N}$, state space $\mathcal{S}$, action space $\mathcal{A}$, transition kernel $P(s' \mid s,a)$, reward function $R:\mathcal{S}\times\mathcal{A}\to\mathbb{R}$, and initial state distribution $\rho_o$. A (stochastic) policy $\pi$ is a conditional distribution $\pi(a\mid s)$ over actions given states. The return of a policy $\pi$ is given by:
\begin{equation}
    J(\pi) := \mathbb{E}_\pi\Bigg[\sum_{t=0}^{T-1} R(s_t,a_t)\Bigg].
\end{equation}

Furthermore, we define the advantage function as $A_\pi(s,a) := Q_\pi(s,a) - V_\pi(s).$ For a second policy $\pi'$, we also define the $\pi'$-averaged advantage of $\pi$ at state $s$: $\bar A_\pi^{\pi'}(s) := \mathbb{E}_{a\sim\pi'(\cdot\mid s)}[A_\pi(s,a)]$.

\begin{lemma}[Performance Difference Lemma, Finite Horizon, $\gamma = 1$]\label{lem:pdiff}
Let $\pi$ and $\pi'$ be two policies. Then
\begin{equation}
    J(\pi') - J(\pi) 
= \sum_{t=0}^{T-1} \mathbb{E}_{s\sim d_{\pi',t}} \big[ \bar A_\pi^{\pi'}(s) \big],
\end{equation}
where $d_{\pi,t}(s) := \Pr_{\pi}(s_t = s)$ denotes the time-$t$ state-marginal under $\pi$.
\end{lemma}

\begin{proof}
We start from the identity $Q_\pi(s,a) 
= r(s,a) + \mathbb{E}_{s'\sim P(\cdot\mid s,a)} [V_\pi(s')]$. Rearranging,
\begin{equation}
    r(s,a) 
= Q_\pi(s,a) - \mathbb{E}_{s'} [V_\pi(s')]
= A_\pi(s,a) + V_\pi(s) - \mathbb{E}_{s'} [V_\pi(s')].
\end{equation}
Consider a trajectory $(s_0,a_0,\dots,s_{T-1},a_{T-1})$ generated by policy $\pi'$. Then
\begin{equation}
    \sum_{t=0}^{T-1} r(s_t,a_t)
= \sum_{t=0}^{T-1} \Big( A_\pi(s_t,a_t) + V_\pi(s_t) - \mathbb{E}[V_\pi(s_{t+1}) \mid s_t,a_t]\Big).
\end{equation}
Taking expectation under $\pi'$ and using the law of total expectation,
\begin{equation}
    J(\pi') 
= \mathbb{E}_{\pi'}\Bigg[\sum_{t=0}^{T-1} A_\pi(s_t,a_t)\Bigg]
+ \mathbb{E}_{\pi'}\Bigg[\sum_{t=0}^{T-1} V_\pi(s_t) - V_\pi(s_{t+1})\Bigg],
\end{equation}
where $V_\pi(s_T) := 0$ by definition. The second sum telescopes:
\begin{equation}
    \sum_{t=0}^{T-1} V_\pi(s_t) - V_\pi(s_{t+1}) 
= V_\pi(s_0) - V_\pi(s_T)
= V_\pi(s_0).
\end{equation}
Thus,
\begin{equation}
    J(\pi') = \mathbb{E}_{\pi'}\Bigg[\sum_{t=0}^{T-1} A_\pi(s_t,a_t)\Bigg] + \underbrace{\mathbb{E}_{s_0\sim \rho_o}[V_\pi(s_0)]}_{J(\pi)}.
\end{equation}
Therefore,
\begin{equation}
    J(\pi') - J(\pi)
= \sum_{t=0}^{T-1} \mathbb{E}_{s_t,a_t\sim \pi'}[A_\pi(s_t,a_t)].
\end{equation}
We can rewrite each term as
\begin{equation}
    \mathbb{E}_{s_t,a_t\sim \pi'}[A_\pi(s_t,a_t)]
= \mathbb{E}_{s\sim d_{\pi',t}}\Big[\mathbb{E}_{a\sim\pi'(\cdot\mid s)}[A_\pi(s,a)]\Big]
= \mathbb{E}_{s\sim d_{\pi',t}} [\bar A_\pi^{\pi'}(s)],
\end{equation}
which proves the claimed identity.
\end{proof}

We now bound the difference between the state marginals $d_{\pi',t}$ and $d_{\pi,t}$ in terms of how different the policies are. For $t \ge 0$, we first define the policy-induced transition kernels:
\begin{equation}
    P_\pi(s'\mid s) := \sum_{a} \pi(a\mid s) P(s'\mid s,a), 
\qquad
P_{\pi'}(s'\mid s) := \sum_a \pi'(a\mid s) P(s'\mid s,a).
\end{equation}
Then $d_{\pi,t+1}^\top = d_{\pi,t}^\top P_\pi$ and $d_{\pi',t+1}^\top = d_{\pi',t}^\top P_{\pi'}$.

\begin{lemma}[State-Distribution Shift Bound, Finite Horizon]\label{lem:ssbound}
Let $\pi,\pi'$ be two policies with the same initial state distribution $d_{\pi,0} = d_{\pi',0} = \rho_o$. Then, for all $t=0,\dots,T-1$,

\begin{equation}
    \|d_{\pi',t} - d_{\pi,t}\|_1
\;\le\;
2\sum_{k=0}^{t-1} \mathbb{E}_{s\sim d_{\pi,k}}
\Big[ D_{\mathrm{TV}}\big(\pi(\cdot\mid s),\pi'(\cdot\mid s)\big) \Big].
\end{equation}
\end{lemma}
\begin{proof}
Define the difference vector $\delta_t := d_{\pi',t} - d_{\pi,t}$. Then:
\begin{align}
\delta_{t+1}
&= d_{\pi',t+1} - d_{\pi,t+1} \nonumber \\ 
&= d_{\pi',t} P_{\pi'} - d_{\pi,t} P_\pi \nonumber \\
&= (d_{\pi',t} - d_{\pi,t})P_{\pi'} + d_{\pi,t}(P_{\pi'} - P_\pi) \nonumber \\
&= \delta_t P_{\pi'} + d_{\pi,t}(P_{\pi'} - P_{\pi}).
\end{align}

Since $P_{\pi'}$ is row-stochastic, $\|\delta_t P_{\pi'}\|_1 \le \|\delta_t\|_1$. Next, we bound the term $d_{\pi,t}(P_{\pi'} - P_\pi)$. Let $w := d_{\pi,t}(P_{\pi'} - P_\pi)$, so $w(s') = \sum_s d_{\pi,t}(s)\big(P_{\pi'}(s'\mid s) - P_\pi(s'\mid s)\big)$. Then:
\begin{align}
\|w\|_1
&= \sum_{s'} |w(s')|
= \sum_{s'} \Big|\sum_s d_{\pi,t}(s)\big(P_{\pi'}(s'\mid s) - P_\pi(s'\mid s)\big)\Big| \nonumber \\
&\le \sum_{s'} \sum_s d_{\pi,t}(s)\big|P_{\pi'}(s'\mid s) - P_\pi(s'\mid s)\big| \nonumber \\
&= \sum_s d_{\pi,t}(s)\sum_{s'}\big|P_{\pi'}(s'\mid s) - P_\pi(s'\mid s)\big|  \nonumber \\
&= \sum_s d_{\pi,t}(s)\,\|P_{\pi'}(\cdot\mid s) - P_\pi(\cdot\mid s)\|_1.
\end{align}
For each fixed $s$, using $P_{\pi'}(s'\mid s) - P_\pi(s'\mid s) 
= \sum_a (\pi'(a\mid s)-\pi(a\mid s))P(s'\mid s,a)$ and the fact that $\sum_{s'}P(s'\mid s,a)=1$, we obtain:
\begin{align}
\|P_{\pi'}(\cdot\mid s) - P_\pi(\cdot\mid s)\|_1
&= \sum_{s'} \Big|\sum_a(\pi'(a\mid s) - \pi(a\mid s))P(s'\mid s,a)\Big| \nonumber \\
&\le \sum_{s'}\sum_a |\pi'(a\mid s) - \pi(a\mid s)|P(s'\mid s,a) \nonumber \\
&= \sum_a |\pi'(a\mid s) - \pi(a\mid s)| \nonumber  \\
&= 2D_{\mathrm{TV}}(\pi(\cdot\mid s),\pi'(\cdot\mid s)).
\end{align}
Hence $\|w\|_1
\le 2\sum_s d_{\pi,t}(s) D_{\mathrm{TV}}(\pi(\cdot\mid s),\pi'(\cdot\mid s))$. Combining these two bounds and using the triangle inequality,
\begin{align}
\|\delta_{t+1}\|_1
&= \|\delta_t P_{\pi'} + d_{\pi,t}(P_{\pi'} - P_\pi)\|_1 \nonumber \\
&\le \|\delta_t P_{\pi'}\|_1 + \|d_{\pi,t}(P_{\pi'} - P_\pi)\|_1 \nonumber  \\
&\le \|\delta_t\|_1 + 2\alpha_t.
\end{align}
By definition, $d_{\pi',0} = d_{\pi,0}$, so $\delta_0 = 0$ and $\|\delta_0\|_1 = 0$. Unrolling the recursion:
\begin{equation}
    \|\delta_t\|_1 
\le 2\sum_{k=0}^{t-1} \mathbb{E}_{s\sim d_{\pi,k}}
\Big[ D_{\mathrm{TV}}\big(\pi(\cdot\mid s),\pi'(\cdot\mid s)\big) \Big].
\end{equation}
\end{proof}

We now define a surrogate objective based on the reference policy $\pi$ and the state distributions $d_{\pi,t}$.

\begin{lemma}[Surrogate--True Performance Gap, Finite Horizon] \label{lemma:gap-finite}
For any policies $\pi$ and $\pi'$, with $D_{\mathrm{KL}}(\pi\Vert\pi')$ the average forward KL under $d_\pi$,
\begin{equation}
    J(\pi') \ \ge\ L_\pi(\pi') \ -\ C\,\sqrt{D_{\mathrm{KL}}(\pi\Vert\pi')},\qquad
    C 
:= T\sqrt{\frac{(T-1)(2T-1)}{3}}\epsilon,
\end{equation}
where $|A^\pi(s,a)|\le \epsilon$ with $\epsilon$ finite, and $L_\pi(\pi') 
:= J(\pi) + \sum_{t=0}^{T-1} \mathbb{E}_{s\sim d_{\pi,t}} \big[ \bar A_\pi^{\pi'}(s) \big].$
\end{lemma}
\begin{proof}
By Lemma \ref{lem:pdiff}, $J(\pi') - J(\pi)
= \sum_{t=0}^{T-1} \mathbb{E}_{s\sim d_{\pi',t}} [\bar A_\pi^{\pi'}(s)]$. Subtracting the surrogate:
\begin{align}
J(\pi') - L_\pi(\pi')
&= \sum_{t=0}^{T-1} 
\Big( \mathbb{E}_{s\sim d_{\pi',t}} \bar A_\pi^{\pi'}(s)
 - \mathbb{E}_{s\sim d_{\pi,t}} \bar A_\pi^{\pi'}(s)\Big) \nonumber \\
&= \sum_{t=0}^{T-1} \sum_{s} \big(d_{\pi',t}(s) - d_{\pi,t}(s)\big)\bar A_\pi^{\pi'}(s).
\end{align}
Taking absolute values and using $|\bar A_\pi^{\pi'}(s)| \le \epsilon$ and applying Lemma \ref{lem:ssbound}:
\begin{align}
    \big|J(\pi') - L_\pi(\pi')\big|
\le \sum_{t=0}^{T-1} \epsilon \|d_{\pi',t} - d_{\pi,t}\|_1 &\le \epsilon \sum_{t=0}^{T-1} 2\sum_{k=0}^{t-1}\mathbb{E}_{s\sim d_{\pi,k}} \Big[ D_{\mathrm{TV}}\big(\pi(\cdot\mid s),\pi'(\cdot\mid s)\big) \Big] \nonumber \\
&= 2\epsilon \sum_{k=0}^{T-1} \mathbb{E}_{s\sim d_{\pi,k}} \Big[ D_{\mathrm{TV}}\big(\pi(\cdot\mid s),\pi'(\cdot\mid s)\big) \Big] \sum_{t=k+1}^{T-1} 1 \nonumber \\
&= 2\epsilon \sum_{k=0}^{T-1}(T-1-k)\,\mathbb{E}_{s\sim d_{\pi,k}} \Big[ D_{\mathrm{TV}}\big(\pi(\cdot\mid s),\pi'(\cdot\mid s)\big) \Big].
\end{align}
For the KL-based bound, we use Pinsker's inequality and Jensen's inequality. For each $t$:
\begin{align}
    \mathbb{E}_{s\sim d_{\pi,t}} D_{\mathrm{TV}}(\pi(\cdot\mid s),\pi'(\cdot\mid s))
\le \mathbb{E}_{s\sim d_{\pi,t}}
\sqrt{\tfrac12 D_{\mathrm{KL}}(\pi(\cdot\mid s)\,\|\,\pi'(\cdot\mid s))} \nonumber \\
\le \sqrt{\tfrac12 \mathbb{E}_{s\sim d_{\pi,t}}
\Big[D_{\mathrm{KL}}(\pi(\cdot\mid s)\,\|\,\pi'(\cdot\mid s))\Big]}.
\end{align}
For conciseness, we define $D_{k} := D_{\mathrm{KL}}(\pi(\cdot\mid s)\,\|\,\pi'(\cdot\mid s))$. Then:
\begin{equation}
    \big|J(\pi') - L_\pi(\pi')\big|
\le 
2\epsilon \sum_{k=0}^{T-1}(T-1-k)\sqrt{\tfrac12 D_k}
= \sqrt{2}\,\epsilon \sum_{k=0}^{T-1} b_k \sqrt{D_k},
\end{equation}
where we have set $b_k := T-1-k$. By Cauchy--Schwarz,
\begin{equation}
    \sum_{k=0}^{T-1} b_k\sqrt{D_k}
\le \sqrt{\sum_{k=0}^{T-1} b_k^2}\,\sqrt{\sum_{k=0}^{T-1} D_k}.
\end{equation}
We note that
\begin{equation}
    \sum_{k=0}^{T-1} b_k^2 = \sum_{j=0}^{T-1} j^2 = \frac{(T-1)T(2T-1)}{6},
\qquad
\sum_{k=0}^{T-1} D_k = T\,\bar D_{\mathrm{KL}}.
\end{equation}
Therefore
\begin{align}
\big|J(\pi') - L_\pi(\pi')\big|
&\le \sqrt{2}\,\epsilon
\sqrt{\frac{(T-1)T(2T-1)}{6}}\,
\sqrt{T\,\bar D_{\mathrm{KL}}} \nonumber\\
&= T\sqrt{\frac{(T-1)(2T-1)}{3}}\,\epsilon\,\sqrt{\bar D_{\mathrm{KL}}}.
\end{align}
\end{proof}
The proof of Theorem \ref{thm:capo-certified-main} for the finite-horizon setting follows exactly the one in Appendix \ref{app:capo}, but applying Lemma \ref{lemma:gap-finite} instead of Lemma \ref{lem:trpo-gap}.

\textbf{Infinite-Horizon vs.\ Finite-Horizon bounds.}
We highlight that, in both settings, the final guarantee takes the same form
$
J(\pi_{\boldsymbol{\theta} + \Delta\boldsymbol{\theta}}) - J(\pi_{\boldsymbol{\theta}})
\;\ge\; \omega - C \sqrt{\delta_F},
$
where \(C = \frac{2\gamma}{(1-\gamma)^2}\,\epsilon\,\sqrt{2}\) for the infinite-horizon case, and 
\(C = T\sqrt{\frac{(T-1)(2T-1)}{3}}\,\epsilon\) for the finite-horizon case.
In both cases, the constant \(C\) scales as \(\mathcal{O}(H_{\mathrm{eff}}^{2})\), where \(H_{\mathrm{eff}}\) denotes the effective horizon: \(H_{\mathrm{eff}} = T\) in the finite-horizon setting, and \(H_{\mathrm{eff}} = \frac{1}{1-\gamma}\) in the infinite-horizon setting.
Practically, this implies that both bounds are equally tight within their respective regimes.

\clearpage

\section{A Closer Look at Model Estimates $\hat{m}_{F}$ and the KL Policy Shift}

In this section, we analyze the relationship between the model’s estimate of directional Fisher curvature, $\hat{m}_F$, and the actual policy shift induced by an update, measured by $D_{\mathrm{KL}}(\pi_{\boldsymbol{\theta}} || \pi_{\boldsymbol{\theta}+\Delta\boldsymbol{\theta}})$. Our goals are two-fold: (i) to clarify what CAPO requires from the underlying model in order to approximate a trust-region and to assess how well this approximation holds, and (ii) to examine the impact of CAPO’s updates on the true change in policy.

\textbf{Does CAPO require a fully calibrated model?} Although well-calibrated estimates are a \textit{sufficient} condition for CAPO’s data-selection mechanism to function effectively, they are not \textit{necessary}. To illustrate this, consider a simple case where the estimated directional Fisher curvature satisfies $\hat{m}_{F} = \alpha \bar D_{\mathrm{KL}}(\pi_{\boldsymbol{\theta}} || \pi_{\boldsymbol{\theta}+\Delta\boldsymbol{\theta}})$, $\alpha > 0$,  where $\alpha >> 1$ or $\alpha << 1$. Such a model is clearly miscalibrated, yet it preserves a strong correlation with the true policy shift. In CAPO, if we aim to enforce the trust-region condition $\bar D_{\mathrm{KL}}(\pi_{\boldsymbol{\theta}} ,|, \pi_{\boldsymbol{\theta}+\Delta\boldsymbol{\theta}}) < \delta$, we can simply set the Fisher-threshold to $\delta_F = \alpha \delta$, which recovers the desired constraint. More generally, CAPO only requires that the estimates be \textit{monotonically correlated} with the true policy change, so that large prospective shifts (those most likely to trigger instability or collapse) are reliably identified.

A natural way to evaluate the quality of the model’s estimates is to measure their correlation with the true policy changes. Although we do not have direct access to this quantity, we can estimate it via samples. In particular, the KL divergence can be reliably estimated using a standard Monte Carlo estimator, which has manageable variance and leverages token-level information. We therefore compute these estimates and report the resulting Spearman correlations in the Table \ref{tab:corr}, where $\hat{m}_F$ is evaluated under both GRPO and CAPO updates at both token and global level. We find that the model estimates exhibit a moderately strong correlation with the actual policy change, indicating a consistent monotonic relationship. Notably, this correlation remains high under both GRPO and CAPO, suggesting that the estimates are meaningful even when they are not used to intervene in the update.

\begin{table}[h]
\centering
\small
\begin{tabular}{l|c|c}
 \textbf{Estimate} &  \textbf{$\rho$ (GRPO)} & \textbf{$\rho$ (CAPO)} \\
\hline
$\hat{m}_{F}$ (Token) & 0.622 & 0.459 \\
$\hat{m}_{F}$ (Global) & 0.596 & 0.498 \\
\end{tabular}
\caption{\textbf{Spearman correlations $\rho$ between Fisher directional curvature estimates $\hat{m}_{F}$ and the estimated policy change $\bar D_{\mathrm{KL}}(\pi_{\boldsymbol{\theta}} || \pi_{\boldsymbol{\theta}+\Delta\boldsymbol{\theta}})$}. We report correlations for both GRPO and CAPO updates. The results indicate that the estimates $\hat{m}_F$ maintain a consistent monotonic relationship with the true policy shift across algorithms, reliable identifying the scale of the policy shifts}. 
\label{tab:corr}
\end{table}

\begin{wrapfigure}{r}{0.5\textwidth}
\vspace{-25pt}
\begin{center}
\includegraphics[width=0.49\textwidth]{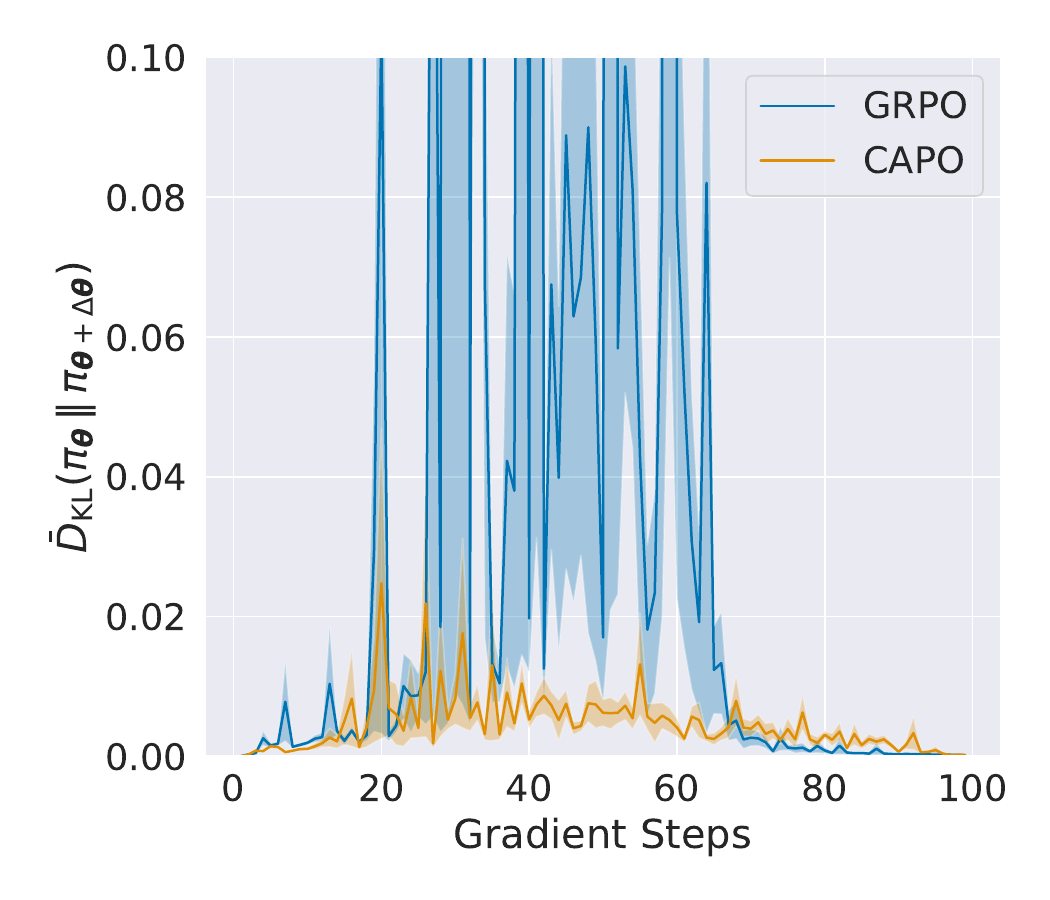}
\captionsetup{skip=-1pt}
\caption{\textbf{Estimated policy KL shifts during training.} GRPO exhibits frequent sharp spikes in policy divergence, indicative of unstable updates, whereas CAPO maintains consistently small shifts, reflecting its ability to enforce trust-region–like behavior throughout training.}
\label{fig:capo_bound}
\end{center}
\end{wrapfigure}

\textbf{Ultimately, does CAPO induce a bound on the true $D_{\mathrm{KL}}(\pi_{\boldsymbol{\theta}} || \pi_{\boldsymbol{\theta}+\Delta\boldsymbol{\theta}})$?} In Figure \ref{fig:capo_bound}, we present the policy shifts over the course of training for both algorithms. GRPO frequently presents peaked shifts, which are often associated with unstable or overly aggressive updates. In contrast,  CAPO generally maintains stable, small shifts, suggesting that it is effective in practically implementing a trust-region behavior throughout training.

\clearpage

\section{Further Questions}

This Appendix presents additional clarification questions aimed at improving the understanding of  the proposed method and experiments. These questions were raised during the peer-review process, and we refer to the OpenReview page for the full discussion.

\textbf{What is the effect of token selection in the sample efficiency evaluation?} In Figure \ref{fig:se_tokens}, we plot the accuracy curves (analogous to Figs. 1 and 2) as a function of the accepted tokens. We observe that these curves closely resemble those obtained when accuracy is plotted against the number of completions. This suggests that the effect of masking on the total number of generated (and accepted) tokens is small, consistent with the rejection rates reported in Figure 5. It also indicates that the learned policies behave similarly in terms of token generation, showing that CAPO improves training sample efficiency without incurring additional inference-time costs.

 \begin{figure}[!htpb]
 \setlength{\abovecaptionskip}{0pt} 
\begin{center}
\includegraphics[width=1.0\textwidth]{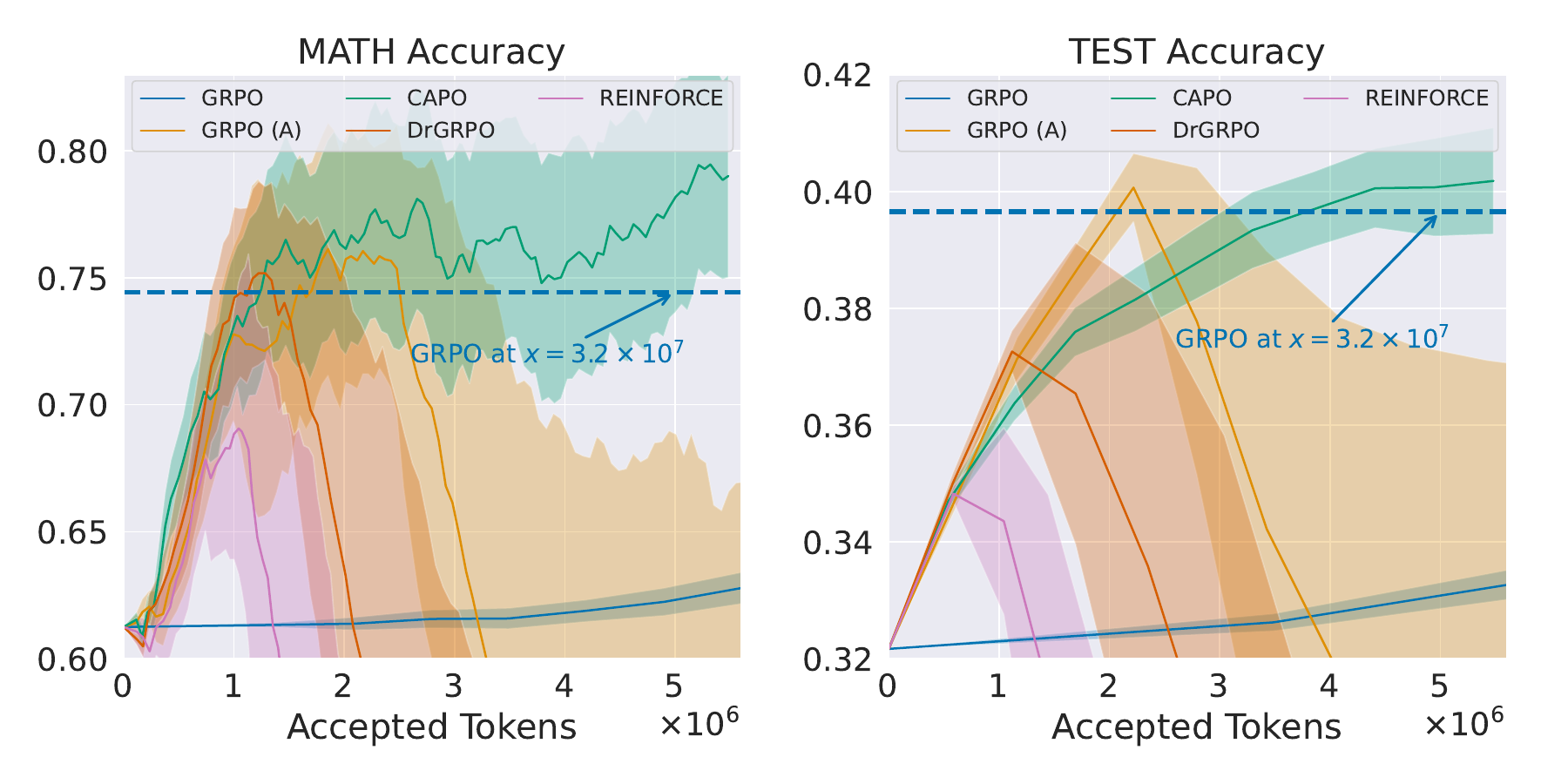}
\caption{\textbf{Sample efficiency curves as a function of the number of accepted tokens}. The trends closely match those obtained when using the number of completions, indicating that masking has minimal impact on token generation and that CAPO improves sample efficiency without added inference cost.}
\label{fig:se_tokens}
\end{center}
\end{figure}

\textbf{What are the similarities between CAPO and TRPO? What are the differences?} In terms of similarities, both CAPO and TRPO share the same motivation: devise a conservative optimization procedure that implements a safe optimization region, typically expressed as a KL ball constraint. This idea predates TRPO, with its roots in natural gradient methods from optimization literature \cite{6790500, 10.5555/2998828.2998935}. What both CAPO and TRPO do is to devise practical instantiations of the natural gradient that is suitable for their respective problem settings.

Methodologically, TRPO incorporates \textit{only} the Fisher matrix in its updates, relying on a first-order approximation of the objective. In contrast, CAPO additionally leverages second-order curvature information of the objective through its Hessian, as shown in Equation \ref{eq:taylor} and further incorporated in the theoretical development in Equation \ref{eq:sec_exp}. The main difference, however, lies in the implementation, which crucially leads to different scalability properties. 

TRPO incorporates the Fisher matrix by employing a Conjugate-Gradient (CG) algorithm to approximate the natural gradient step without fully materializing the Fisher matrix. Then, TRPO employs a line search algorithm to solve the constrained optimization problem. The CG algorithm involves maintaining five vectors of size $d$ (the gradient, current iterate, the residual, the search direction, and the matrix-vector buffer), where $d$ is the number of parameters in the policy. While this memory cost is feasible for small deep networks (as usual in traditional Deep RL research), it is prohibitive for LLM scale, where $d$ is in the billions.

Furthermore, the CG algorithm is iterative, and each iteration costs roughly the same as a backward pass, unless you sacrifice your Fisher matrix estimation by subsampling data. TRPO uses ten iterations. Considering the execution time in our setup (Appendix \ref{app:compcost}), this overhead is also prohibitive. Lastly, the line search algorithm requires $M$ additional forward passes in the whole batch ($M$ is the number of search trials), which is also a substantial cost in our setup (also illustrated in Appendix \ref{app:compcost}). Overall, TRPO's memory and execution costs are prohibitive to LLM scale. CAPO, in contrast, leverages the last layer model and the optimizations described in Section 4.1, resulting in much lower costs, as evaluated in Table \ref{tab:compcost} of Appendix \ref{app:compcost}.

In summary, while TRPO and CAPO share the same motivation and draw from the same seminal work on natural gradients, CAPO offers a formulation that scales to the memory and compute demands of LLM policies.

\end{document}